\documentclass[twocolumn]{IEEEtran}
\IEEEoverridecommandlockouts                              
\bstctlcite{IEEEexample:BSTcontrol}

\usepackage[utf8]{inputenc} 
\usepackage[T1]{fontenc}    
\usepackage{hyperref}       
\usepackage{url}            
\usepackage{booktabs}       
\usepackage{amsfonts}       
\usepackage{nicefrac}       
\usepackage{comment}
\usepackage{lmodern}
\usepackage{scalefnt}
\usepackage{setspace}
\usepackage{romannum}
\usepackage[dvipsnames]{xcolor}
\usepackage{tkz-berge}
\usetikzlibrary{fit,shapes}
\usepackage{calrsfs}
\DeclareMathAlphabet{\pazocal}{OMS}{zplm}{m}{n}
\usepackage{graphicx}
\usepackage{graphics}
\usepackage{epsfig}
\usepackage{mathptmx}
\usepackage{caption}
\usepackage{subcaption}
\usepackage{tikz}
\usetikzlibrary{positioning}
\usepackage{amsmath, amssymb}
\usetikzlibrary{arrows}
\usepackage{dsfont}
\newcommand{\indic}{\mathds{1}}
\usepackage{amsthm}
\usepackage{graphicx}
\DeclareMathOperator*{\argmax}{arg\,max}
\DeclareMathOperator*{\argmin}{arg\,min}

\usepackage{kantlipsum}
\usepackage{float}
\usepackage{amsfonts}
\usepackage{setspace}
\usepackage{multirow}
\usepackage[english]{babel}
\usepackage{float}
\usepackage{algorithmicx}
\usepackage[section]{algorithm}
\usepackage{algpascal}
\usepackage{algc}
\usepackage{algcompatible}
\usepackage{algpseudocode}
\let\oldReturn\Return
\renewcommand{\Return}{\State\oldReturn}
\usepackage{mathtools}
\usepackage{mathptmx}
\usepackage{mathtools, cuted}
\usepackage{textcomp}
\usepackage[english]{babel}
\usepackage{rotating}
\usepackage{pgfplots}
\pgfplotsset{compat=1.5}
\usepackage{siunitx}
\usepackage{pgfplotstable}
\usepackage{bbold}
\usepackage{float}

\newtheorem{theorem}{Theorem}

\newtheorem{lemma}{Lemma}
\newtheorem{prop}{Proposition}

\newtheorem{defi}{Definition}
\newtheorem{assume}{Assumption}

\newtheorem{remark}{Remark}

\renewcommand{\footnotesize}{\scriptsize}

\DeclareMathAlphabet{\pazocal}{OMS}{zplm}{m}{n}

\newcommand{\SE}{\mathrm{SD}}  
\newcommand{\pS}{{\pazocal{S}}}

\newcommand{\X}{\pazocal{X}}

\newcommand{\pR}{\pazocal{R}}
\newcommand{\pG}{{\pazocal{G}}}
\newcommand{\N}{\pazocal{N}}
\newcommand{\bR}{\mathbb{R}}
\newcommand{\bE}{\mathbb{E}}

\renewcommand{\l}{\ell}
\newcommand{\wB}{\widehat{B}}
\newcommand{\wW}{\widehat{W}}
\newcommand{\ww}{\widehat{w}}
\newcommand{\Thetahat}{\widehat{\Theta}}
\newcommand{\Thetas}{\Theta^\star}

\newcommand{\GradB}{\mathrm{GradB}}
\newcommand{\NSR}{\mathrm{NSR}}
\newcommand{\svdeq}{\overset{\mathrm{SVD}}=} 
\newcommand{\sigmin}{{\sigma_{\min}^\star}}
\newcommand{\sigmax}{{\sigma_{\max}^\star}}
\newcommand{\Var}{\mathrm{Var}}

\newcommand{\bnt}{{b_{\nt}}}
\newcommand{\rbnt}{{r_{\bnt}}}
\newcommand{\kbnt}{{\kappa_{\bnt}}}
\newcommand{\thetats}{\theta_t^\star}
\newcommand{\thetahat}{\widehat{\theta}}
\newcommand{\thetahatmt}{\widehat{\theta}_{m-1, t}}
\newcommand{\thetahatt}{\widehat{\theta}_t}
\usepackage{dsfont}
\newcommand{\delt}{\delta_\l}
\newcommand{\delto}{\delta_{\l + 1}}
\newcommand{\nt}{{n, t}}
\newcommand{\NT}{{N, T}}
\newcommand{\xnt}{x_\nt}

\newcommand{\pxtc}{x_\nt}
\newcommand{\pb}{x_{b_{n, t}}}
\newcommand{\norm}[1]{\left\lVert#1\right\rVert}
\newcommand{\sm}{{\scalebox{.6}{(m)}}}
\newcommand{\szero}{{\scalebox{.6}{(0)}}}
\newcommand{\sone}{{\scalebox{.6}{(1)}}}
\newcommand{\smo}{{\scalebox{.6}{(m-1)}}}

\newcommand{\ymt}[2]{Y_{#1}^{{\scalebox{.6}{(#2)}}}}
\newcommand{\phimt}[2]{\Phi_{#1}^{{\scalebox{.6}{(#2)}}}}
\newcommand{\etamt}[2]{\eta_{#1}^{{\scalebox{.6}{(#2)}}}}
\newcommand{\Bm}[1]{\wB^{\scalebox{.6}{(#1)}}}
\newcommand{\Wm}[1]{\wW^{\scalebox{.6}{(#1)}}}

\usepackage{wrapfig}
\usepackage{cite}

\definecolor{deepblue}{HTML}{1f77b4}
\definecolor{orange}{HTML}{ff7f0e}
\definecolor{deepgreen}{HTML}{2ca02c}
\definecolor{deepred}{HTML}{d62728}

\DeclareRobustCommand{\legendsquare}[1]{%
  \textcolor{#1}{\rule{1ex}{1ex}}%
}

\newcommand{\cblue}{\color{black}}
\newcommand{\cred}{\color{black}}

\newcommand\scalemath[2]{\scalebox{#1}{\mbox{\ensuremath{\displaystyle #2}}}}

\title{\LARGE \bf Multi-Task Representation Learning for {\cblue Conservative} Linear Bandits}
\author{Jiabin Lin and Shana Moothedath
\thanks{J. Lin is with the Department of Computer Science and Technology, Qingdao University, China, and S. Moothedath is with Department of Electrical and Computer Engineering, Iowa State University, USA. Email: linjiabin@qdu.edu.cn,mshana@iastate.edu. A prior related work was published in the 2024 International Conference on Machine Learning (ICML, 2024) \cite{OurICML}. This work was supported by the U.S. National
Science Foundation under Grant NSF CAREER 2440455.}
}

\begin{document}
\pagenumbering{arabic}
\maketitle

\begin{abstract}
This paper presents the Constrained Multi-Task Representation Learning (CMTRL) framework for linear bandits. {\cblue We consider $T$ linear bandit tasks in a $d$ dimensional space, which share a common low-dimensional representation of dimension $r \ll \min\{d, T\}$.} 
Furthermore, tasks are constrained so that only actions meeting specific safety or performance requirements are allowed, referred to as {\cblue conservative} (safe) bandits. We introduce a novel algorithm, Safe-Alternating projected Gradient Descent and minimization (Safe-AltGDmin), to recover a low-rank feature matrix while satisfying the given constraints. Building on this algorithm, we propose a multi-task representation learning framework for {\cblue conservative} linear bandits and establish theoretical guarantees for its regret and sample complexity bounds. We presented experiments and compared the performance of our algorithm with benchmark algorithms.
\end{abstract}
\section{Introduction}
Stochastic linear bandits are a fundamental online learning framework in which an agent sequentially selects actions to maximize cumulative rewards while balancing exploration and exploitation \cite{lattimore2020bandit}. They arise in applications such as robotics, clinical trials, communications, and recommender systems.
%
Recently, multi-task learning, which enables the simultaneous learning of multiple related tasks, has garnered significant attention due to its ability to leverage shared structures across tasks to enhance performance {\cblue \cite{Jiabin_neurips, duong2025beyond}.} 
A common assumption in multi-task learning is that different tasks share an underlying representation. This shared structure enables the model to leverage information across tasks, referred to as {\em multi-task representation learning (MTRL)}, improving overall learning efficiency and performance \cite{du2020few, collins2021exploiting, tripuraneni2021provable, cella2021multi, thekumparampil2021sample}.

The MTRL approach is particularly beneficial in data-scarce settings, as it allows models to improve efficiency by sharing representations. 
Motivated by the immense success of MTRL in the supervised learning (static data) setting, recent works have begun exploring its potential in sequential decision-making problems, including bandit learning \cite{Jiabin_neurips, duong2025beyond,yang2020impact, cella2023multi,cella2021multi,hu2021near,du2023multi, OurICML} and reinforcement learning {\cblue \cite{ishfaq2024offline, nagaraj2023multi, borsa2016learning,lu2022provable}.}
The primary goal of these works is to learn shared representations that enable effective learning across tasks, enhancing performance with limited data and reducing overall sample complexity.
However, in many practical real-world scenarios, such as autonomous driving, power systems, and finance, there are often safety or performance constraints that must be considered. Here, each action must meet certain performance/safety thresholds, particularly in safety-critical applications. While recent studies have addressed conservative or safe bandits (single-task setting)  {\cblue \cite{gangrade2025feasible, moradipari2020stage, kazerouni2017conservative, khezeli2020safe, camilleri2022active},} works on distributed (multi-task) bandit learning under safety constraints remains limited \cite{lin2024distributed}. Representation learning in constrained multi-task bandits is underexplored and is the main focus of this paper.

This paper presents a novel framework {\em Constrained Multi-Task Representation Learning (CMTRL)}, where the goal is to learn shared representations across multiple tasks while adhering to predefined safety or performance constraints.
By addressing both learning and constraint satisfaction, CMTRL extends the benefits of multi-task learning to constrained learning.
To the best of our knowledge, this is the first work to consider this problem and to investigate it with provable guarantees. The key contributions are presented below.

\noindent{\bf Main Contributions:}
We present a new algorithmic framework, the Safe-Alternating gradient descent and minimization (Safe-AltGDmin) algorithm, to solve the CMTRL problem with a low-rank common representation and present regret and sample complexity bounds. 
Our algorithm consists of three main components: safe exploration and safety estimation, greedy exploration, and reward estimation.
It starts with safe exploration based on a known baseline policy without any model information
Subsequently, the algorithm utilizes estimates of unknown reward parameters learned using the AltGDmin approach to choose actions via a greedy strategy. In contrast to previous work on {\cblue conservative} bandits \cite{lin2025distributed, moradipari2020stage}, which eliminates unsafe actions, our estimator-based approach establishes bounds on safe exploration rounds and GD iterations, facilitating efficient safe learning. The Safe-AltGDmin algorithm is based on gradient descent (GD), making it computationally efficient. Due to the non-convex nature of the cost function, we employ a careful initialization for GD referred to as spectral initialization. Specifically, we show that with carefully chosen learning parameters, all actions taken during the safe exploration phase and the greedy exploration phase adhere to the constraints. Under certain sample complexity conditions, we illustrate that the subspace distance (SD) between the shared model and the learned model decreases exponentially with high probability.
Finally, we present a bound on the cumulative regret, which demonstrates a sublinear relation with respect to the number of tasks. 
To verify our theoretical findings, we performed comprehensive experiments utilizing synthetic data and real Movielens datasets. The experimental results demonstrate that our algorithm consistently satisfies the constraints and outperforms existing benchmark algorithms.

In summary, our key contributions are fourfold. (1)~First, we formulated the CMTRL framework, establishing a foundational approach for addressing multi-task problems in safety-constrained multi-task bandits. (2)~Second, we developed a novel Safe-AltGDmin algorithm that enhances the learning process across multiple tasks, improving both learning efficiency and performance. (3)~Third, we established the convergence and regret guarantees of the algorithm and sample efficiency. (4)~Finally, we conducted extensive experiments to validate our theoretical results, showcasing the practical applicability and performance of our approach.

\noindent{\bf Related Work:}
Low-dimensional representation learning has emerged as a popular research area in supervised \cite{du2020few, tripuraneni2021provable, collins2021exploiting} and dynamic settings \cite{hu2021near, yang2020impact, cella2023multi, OurICML}. 
Solving multi-task representation learning in linear bandits inherently presents a non-convex estimation challenge. Existing studies often either assume that the optimal solution to the non-convex problem is known \cite{yang2020impact, hu2021near}, or rely on the trace norm regularization approach \cite{cella2023multi}. 
Our prior work \cite{OurICML} introduced a new approach, alternating gradient descent and minimization algorithm (AltGDmin), for estimating the unknown reward matrix. This work, however, considered an unconstrained setting. In this paper, we propose a novel algorithm, Safe-AltGDmin for constrained {\cblue multi-task} representation learning in {\cblue linear} bandits. Specifically, we formulate the MTRL problem with explicit safety constraints and develop an approach that guarantees the safety of all chosen actions throughout the learning process. See Table~\ref{table3}.

Constrained bandit learning has received significant attention, particularly in the context of incorporating safety into learning \cite{kazerouni2017conservative, wu2016conservative, moradipari2020stage, pacchiano2024contextual, khezeli2020safe, amani2019linear, amani2020generalized, camilleri2022active, varma2023stochastic, sun2017safety, zhou2022kernelized, jagerman2020safe, daulton2019thompson, garcelon2020improved, chen2022doubly, liu2021efficient}. 
{\cred \cite{moradipari2020stage} studied stage-wise conservative linear bandits in the single-task setting and developed algorithms that maintain performance constraints by constructing a safe action set. However, it does not consider multiple related tasks or a shared low-rank representation across tasks.}
Recently, \cite{lin2025distributed} considered constrained learning for multi-task bandits. However, the approach there did not address the representation learning setting or sample complexity that is studied in this paper.
In addition to safety constraints, some studies have examined MABs under budget constraints, where the objective is to maximize cumulative reward without exceeding a defined budget, as explored in works like \cite{badanidiyuru2018bandits, badanidiyuru2014resourceful, agrawal2014bandits, wu2015algorithms}.
Resource-constrained bandits are further examined in \cite{slivkins2024contextual, sivakumar2022smoothed, deb2024think}. Constrained bandits in the context of 
nonlinear bandits are studied in \cite{deb2024conservative}.
Multi-task representation learning for constrained bandits remains largely unexplored in the existing literature, which is the focus of this paper.
%
\vspace{-2 mm}
{\cblue
\begin{table*}[h]
\caption{Table comparing technical contributions of existing work versus ours.}
\label{table3}
\begin{center}
\resizebox{0.99\linewidth}{!}
{
\begin{tabular}{|l|l|l|l|l|l|}\toprule

{\bf Work}
& {\bf MTRL}
& {\bf Constraints}
& {\bf Regret}
& {\bf Sample Complexity (All Tasks)}
& {\bf Sample Complexity (Per Task)}\\
\hline\midrule

Yang et al., 2020~\cite{yang2020impact}
& $\checkmark$
& $\times$
& $\widetilde{O}(T \sqrt{rN} + \sqrt{drNT})$
& $\times$
& $\times$
\\
\hline

Hu et al., 2021~\cite{hu2021near}
& $\checkmark$
& $\times$
& $\widetilde{O}(T \sqrt{drN} + d \sqrt{rNT})$
& $\times$
& $\times$
\\
\hline

Cella et al., 2023~\cite{cella2023multi}
& $\checkmark$
& $\times$
& $\widetilde{O}(T \sqrt{rN} + \sqrt{drNT})$
& $\times$
& $\times$
\\
\hline

{\cred Naive extension of~\cite{moradipari2020stage}}
& $\times$
& $\checkmark$
& {\cred $\widetilde{O}(d T \sqrt{N})$}
& $\times$
& $\times$
\\
\hline

Lin et al., 2025~\cite{lin2024distributed}
& $\times$
& $\checkmark$
& $\widetilde{O}(d \sqrt{NT})$
& $\times$
& $\times$
\\
\hline

Lin et al., 2024~\cite{OurICML}
& $\checkmark$
& $\times$
& $\widetilde{O}(\max\{\epsilon, \epsilon_{\mathrm{noise}}\} \sqrt{rNT})$
& $\widetilde{O}((d + T) r (r^2 + \log\frac{1}{\max\{\epsilon, \epsilon_{\mathrm{noise}}\}}))$
& $\widetilde{O}(\max\{\log d, \log T, r\} \log\frac{1}{\max\{\epsilon, \epsilon_{\mathrm{noise}}\}})$
\\
\hline

{\bf This work}
& $\checkmark$
& $\checkmark$
& $\widetilde{O}(\sqrt{rNT (\rho^2 - \rho + 1)})$
& $\widetilde{O}((d + r) r \max\{1, \frac{\NSR}{\epsilon^2}\}))$
& $\widetilde{O}((r + \log T + \log d) \max\{1, \frac{\NSR}{\epsilon^2 r}\})$
\\
\hline

\bottomrule
\end{tabular}
}
\end{center}
\end{table*}
}
\section{Problem Setting}\label{sec:prob}
\noindent{\bf Background: Linear Bandits and {\cblue Conservative} Bandits.} 
\noindent{\em Linear Bandits:} 
We denote the action set as $\X$. At each round $n \in [N]$, the learner chooses an action $x_n \in \X$. The environment  provides a reward $y_n \in \bR$, defined as
$y_n := x_n^\top \theta^\star + \eta_n,$
where $\theta^\star$ represents a fixed but unknown reward parameter, and $\eta_n$ represents the noise following to a zero-mean $\sigma_\eta-$Gaussian distribution. We define $r_n = x_n^\top \theta^\star$ as the expected reward associated with the action $x_n$, i.e., $r_n = \bE[y_n]$. The goal of the learner is to maximize the cumulative reward $\sum_{n=1}^N r_n$, which is minimizing the cumulative (pseudo) regret
\begin{equation*}
\pR_N = \sum_{n=1}^N x_n^{\star^\top} \theta^\star - x_n^\top \theta^\star,
\end{equation*}
where $x_n^\star$ denotes the optimal action.

\noindent{\em {\cblue Conservative} linear bandits:} The learner has a baseline policy that provides a baseline action $x_{b_n}$ and a reward $r_{b_n}$ for each round $n$, both of which are known to the learner \cite{moradipari2020stage, chaudhary2022safe, garcelon2020improved, pacchiano2024contextual}. The learner’s action selection rule is subject to a stage-wise constraint
$r_n = x_n^\top \theta^\star \geqslant (1 - \alpha) r_{b_n},$ 
where $\alpha \in (0, 1)$ is the parameter that controls the conservatism level of the learning process, which is known to the learner. In each round, the reward for the chosen action must be no less than the specified fraction $(1 - \alpha)$ of the baseline reward. Therefore, the goal of the learner is to choose actions that maximize the cumulative reward while meeting the stage-wise constraint
\begin{eqnarray*}
\text{minimize~} \pR_N~~\text{such that} \quad x_n^\top \theta^\star \geqslant (1 - \alpha) r_{b_n}, n\in [N].
\end{eqnarray*}
\subsection{Problem Setting}
In this paper, we study the CMTRL problem in linear bandits. In each round $n \in [N]$, each task $t \in [T]$ independently chooses an action $x_\nt$, subject to the safety constraint. The environment provides the feedback $y_\nt$ as
$y_\nt := x_\nt^\top \thetats + \eta_\nt,$
where $\thetats \in \bR^d$ denotes the unknown reward parameter related to task $t$. The goal is to optimize the selection of actions to maximize the cumulative reward or minimize the cumulative regret while ensuring that the constraints are satisfied. 
\begin{eqnarray}\label{eq:main}
\min \pR_{T, N} = \sum_{n=1}^N \sum_{t=1}^T x_\nt^{\star^\top} \thetats - x_\nt^\top \thetats \\ \text{such that~} x_\nt^\top \thetats \geqslant (1 - \alpha) \rbnt, t\in [T],  n \in [N],\nonumber
\end{eqnarray}
where $x_\nt^\star$ denotes the optimal action. We assume that $\Thetas = [\theta_1^\star \cdots \theta_T^\star]$ is a rank-$r$ matrix, where $r \ll \min\{d, T\}$, allowing improved learning by leveraging the similarities of the tasks. 
%
\subsection{Preliminaries}
Let $\Theta^\star \svdeq B^\star {\Sigma V^{\star}} := B^\star  W^\star$ represent the reduced (rank $r$) Singular Value Decomposition (SVD) of $\Thetas$, where $B^\star \in \bR^{d \times r}$, $V^\star \in \bR^{r \times T}$, and $\Sigma \in \bR^{r \times r}$ is a diagonal matrix with non-negative entries (singular values). Here, $B^\star$ and ${V^\star}^\top \in $ are matrices with orthonormal columns {\em (basis matrices)}. We define $W^\star:= \Sigma V^{\star}$. 
We denote $\sigmax$ and $\sigmin$ as the maximum and minimum singular values of $\Sigma$, respectively. The condition number of $\Sigma$ is defined as $\kappa:= \sigmax/\sigmin$. 

\noindent{\bf Notations:} For any positive integer $n$, $[n]$ denotes $\{1, 2, \cdots, n\}$. For any vector $x$, $\norm{x}$ represents the $\ell_2$ norm, while for any matrix $A$, $\norm{A}$ denotes the $2$-norm. The Frobenius norm is denoted as $\norm{A}_F$. The symbol $\top$ represents the transpose of a matrix or vector, while $|x|$ indicates the element-wise absolute value of a vector $x$. The notation $I_k$ (or sometimes just $I$) represents the $k \times k$ identity matrix, while $e_k$ denotes the $k-$th canonical basis vector.
For basis matrices $B_1$ and $B_2$, we define Subspace Distance (SD) as $\SE(B_1, B_2) := \|(I - B_1 B_1^\top) B_2\|_F$. We define the noise-to-signal ratio as $\NSR := \frac{T \sigma_\eta^2}{\sigma_{\min}^{\star^2}}$.

We present the assumptions used in our analysis below. 
\begin{assume}\label{assume:B}
(Common Feature Extractor). We assume the existence of a matrix $B^\star \in \bR^{d\times r}$, denoted as the common feature extractor, along with a set of linear coefficient vectors $\{w_t\}_{t=1}^T$. Hence, the reward parameter matrix can be expressed as $\Theta^\star = B^\star W^\star$, where $W^\star=[w_1^\star, w_2^\star, \ldots, w^\star_T]$. 
\end{assume}
Based on this assumption, the expected reward for the $t$-th task in the $n$-th round is $\bE[y_\nt] = \langle \pxtc, B^\star w_t^\star \rangle$, where $\pxtc\in \bR^d$ denotes the feature vector. The common feature extractor across various tasks facilitates effective learning from limited data for each task. It is consistent with previous works on representation learning \cite{yang2020impact,du2020few, hu2021near, OurICML, tripuraneni2021provable, collins2021exploiting}. 
\begin{assume}\label{assume:bound}
Let $\kbnt := x_\nt^{\star^\top} \thetats - x_{b_\nt}^\top \thetats$ denote the baseline gap. Let constants $\kappa_l$ and $\kappa_h$ are such that $0 \leqslant \kappa_l \leqslant \kbnt \leqslant \kappa_h$ for all $n \in[N]$ and $t \in [T]$. 
\end{assume}
This assumption is non-restrictive and is used in {\cblue (conservative)} linear bandits. Furthermore, the interval constraint on $\kbnt$ helps in limiting the gap in expected rewards between the optimal and baseline action. This assumption is consistent with prior studies, including \cite{moradipari2020stage, lin2024distributed}.  
\begin{assume}\label{assume:iid}
We assume that $\pxtc$ follows an independent and identically distributed (i.i.d.) standard Gaussian. The additive noise variable $\eta_{n, t}$ follows i.i.d. Gaussian distribution with a zero mean and variance $\sigma_\eta^2$. 
\end{assume}
This assumption is standard in the study of random design linear regression \cite{cella2023multi, cella2021multi, OurICML}. 
We note that while the assumption on $\pxtc$ holds for the first epoch in our algorithm (during the random exploration phase), it is restrictive for subsequent epochs. 
We can relax the assumption to accommodate a non-zero mean by rewriting $y_\nt = (\pxtc - \mu_{x_\nt})^{\top} \thetats + \eta_\nt + \mu_{x_\nt} \thetats$, where the term $\eta_\nt + \mu_{x_\nt} \thetats$ now constitutes the noise component. 
A complete relaxation of this assumption is planned as part of our future work.
\begin{assume}\label{assume:incoherence}
We assume the existence of constants $l$ and $u$, where $0 < l \leqslant u$ such that $l \leqslant \|w_t^\star\|_2 \leqslant u$ for all $t \in [T]$.
\end{assume}
This assumption implies the column-wise incoherence of the true reward matrix $\Thetas$, as described in Section~\ref{app:prelim}. This is essential for interpolating across columns based on localized observations $y_\nt$, which rely only on individual columns of $\Thetas$. The incoherence is a crucial attribute necessary for effective matrix estimation and various sensing challenges involving sparse measurements \cite{matcomp_candes, chi2019nonconvex}, and it has been utilized in many recent studies in representation learning \cite{tripuraneni2021provable, collins2021exploiting, thekumparampil2021sample}.

\section{Proposed Safe-AltGDmin Algorithm}\label{sec:alg}
This section presents the Safe-AltGDmin algorithm for CMTRL in linear bandits, building on the AltGDmin framework studied in \cite{lrpr_gdmin,collins2021exploiting, OurICML}. Our approach uses a doubling schedule rule, as suggested in many studies \cite{gao2019batched,han2020sequential,simchi2019phase}, which systematically divides the learning horizon $N$ into $M$ epochs.

An online algorithm cannot learn the optimal policy without sufficient system exploration.
To learn the optimal policy, we face two main challenges: (i) ensure that the chosen actions consistently satisfy the constraints and  (ii) obtain accurate estimations of $\Theta^\star$ for the rank-constrained non-convex optimization problem in Eq.~\eqref{eq:main}.
To address these challenges, our algorithm consists of two main parts: (a) safe exploration and (b) alternating GD and minimization for feature matrix estimation.
In epoch-1, we perform safe exploration using the baseline policy. Following this, we perform an AltGDmin-based estimation to estimate the unknown reward matrix. During the safe exploration, instead of relying solely on the baseline policy, we incorporate conservative exploration techniques motivated by \cite{moradipari2020stage, khezeli2020safe, chaudhary2022safe}. In the subsequent epochs, we refine the estimates and perform a greedy exploration based on the estimate of the parameters from previous epoch.


\subsection{Conservative Safe Exploration}
We assume that the learner has a baseline policy that provides a baseline action $x_\bnt$.
During the first epoch, for each round $n \in [\pG_1]$, each task $t$ selects a baseline action $x_\bnt$. Since relying only on the baseline action restricts exploration, we utilize a conservative feature vector as
$
(1 - \rho) x_\bnt + \rho \zeta_\nt, 
$
where $\zeta_\nt$ represents a sequence of vectors sampled from an i.i.d. standard Gaussian distribution, $\rho \in [0, 1]$.  Qualitatively, the user-specified parameter $\rho$ controls the balance between safety and exploration. In Lemma~\ref{l8} we provide an upper bound on $\rho$ such that for all $\rho \in [0, \frac{\alpha\rbnt}{\rbnt+\sqrt{2\frac{r}{T}\log\frac{\pG_1 T}{\delta}}\mu\sigma_{\max}^\star}]$, conservative action is guaranteed to be safe. 

\subsection{AltGDmin Estimation}
In each epoch $m$, $\pG_m - \pG_{m-1}$ samples are collected during the exploration.
Using the data, our goal is to minimize
\begin{align}
f_m(\wB, \wW) = \sum_{n=\pG_{m-1}+1}^{\pG_m} \sum_{t=1}^T \|y_\nt - x_{n, t}^\top \wB\widehat{w}_t\|^2.\label{eq:cost}
\end{align}
Here $\wB\in \mathbb{R}^{d \times r}$, $\wW =[\ww_1, \ldots, \ww_T]\in \mathbb{R}^{r \times T}$, and the matrix $\Thetahat=\wB \wW$ is an estimate of the parameter $\Thetas$ at epoch $m$.
Due to the non-convex nature of the cost function $f_m(\wB, \wW)$, we utilize a spectral initialization approach \cite{lrpr_gdmin, OurICML}, to derive the top $r$ singular vectors from the initial estimation $\widehat{\Theta}_{0, full}$, where
\begin{align*}
\widehat{\Theta}_{0, full} &= \frac{1}{\pG_1} [({\phimt{1}{1}}^\top \ymt{1}{1}), \cdots, ({\phimt{T}{1}}^\top \ymt{T}{1})] \\
&= \frac{1}{\pG_1} \sum_{t=1}^T \sum_{n=1}^{\pG_1} x_{n, t} y_{n, t} e_t^\top, 
\end{align*}
where $\phimt{t}{1}$ denotes the feature matrix generated by stacking all the feature vectors $\{x_{n, t}\}_{n=1}^{\pG_1}$ for task $t$ in first epoch. The expected value $\bE[\widehat{\Theta}_{0, full}] = \Theta^\star$ and $\widehat{\Theta}_{0, full}$ represents a summand of sub-exponential random variable with a maximum norm $\max_t \|\thetats\| \leqslant \mu \sqrt{\frac{r}{T}} \sigma_{\max}^\star$ based on Assumption~\ref{assume:incoherence} and Section~\ref{app:prelim}. However, the large magnitude presents an issue in bounding $\|\widehat{\Theta}_{0, full} - \Theta^\star\|$ with the desired sample complexity. To reduce the impact of the outlier value and ensure reliable parameter initialization, we use a truncation 
strategy \cite{nayer2020provable, lrpr_gdmin, OurICML}, initializing $\wB^{\szero}$ with the top $r$ left singular vectors of
\begin{align*}
\widehat{\Theta}_{0}=\frac{1}{\pG_1} \sum_{t=1}^T \sum_{n=1}^{\pG_1} x_{n, t} y_{n, t} e_t^\top \indic_{\{y_{t, n}^2 \leqslant \alpha\}}=\frac{1}{\pG_1} \sum_{t=1}^T x_{n, t} y_{t, trunc}(\alpha) e_t^\top,
\end{align*}
\begin{algorithm}[t]
    \caption{Proposed Safe-AltGDmin Algorithm} 
    \label{alg1}
\begin{algorithmic}[1]
    \State Let $M = \lceil \log_2 \log_2 N \rceil$, $\pG_0 = 0$, $\pG_M = N$, $\pG_m = N^{1 - 2^{-m}}$, for $1 \leqslant m \leqslant M - 1$
    \State {\bfseries Parameters:} Multiplier in specifying $\alpha$ for init step, $\tilde{C}$; GD step size, $\gamma$; Number of GD iterations, $L$
    \For{$m \leftarrow 1, \cdots, M$}
        \If{$m=1$}
            \For{$n \leftarrow 1, \cdots, \pG_1$}
                \State For each task $t \in [T]$: choose baseline action $\xnt = x_{b_{n, t}}$, get the feature vector $\pxtc = (1 - \rho) \pb + \rho \zeta_\nt$, and obtain the reward $y_\nt$
            \EndFor
        \Else
            \For{$n \leftarrow \pG_{m-1} + 1, \cdots, \pG_m$}
                \State For each task $t \in [T]$: choose action $\xnt = \argmax_{x\in\X} x^\top \widehat{\theta}_t^{\smo}$, and obtain $y_\nt$\label{line:greedy}
            \EndFor
        \EndIf
        \State For $t \in [T]$, compute $\ymt{t}{m} = [y_{\pG_{m-1} + 1, t}, \cdots, y_{\pG_m, t}]^\top$, $\phimt{t}{m} = [x_{\pG_{m-1} + 1, t}, \cdots, x_{\pG_m, t}]^\top$
        \State {\bf Sample-split:} Partition the measurements and measure matrices into $2 L$ equal-sized disjoint sets for $m \geqslant 2$ and $2 L + 1$ sets for $m = 1$. Denote these by $\ymt{t, \tau}{m}$, $\phimt{t, \tau}{m}$, $\tau = 00$ (only for $m = 1$), $01, \cdots 2 L$
        \If{$m=1$} {\bfseries Spectral Initialization}
            \State Use $\ymt{t}{1} \equiv \ymt{t, 00}{1}$, $\phimt{t}{1} \equiv \phimt{t, 00}{1}$, $\alpha = \frac{\tilde{C}}{\pG_1 T} \sum_{n=1, t=1}^{\pG_1, T} y_{n, t}^2$
            \State $y_{t, trunc}(\alpha) := \ymt{t}{1} \circ \indic_{\{|\ymt{t}{1}| \leqslant \sqrt{\alpha}\}}$
            \State $\widehat{\Theta}_0 := \frac{1}{\pG_1} \sum_{t=1}^T {\phimt{t}{1}}^\top y_{t, trunc}(\alpha) e_t^\top$
            \State Set $\Bm{0} \leftarrow \text{top-}r\text{-singular-vectors of} \; \widehat{\Theta}_0$
        \EndIf
        \State {\bfseries GD-Minimization}
        \State Set $B_0 \leftarrow \Bm{m-1}$
        \For{$\l = 1$ to $L$}
            \State Let $B \leftarrow B_{\l-1}$
            \State {\bfseries Update $w_{t, \l}, \theta_{t, \l}$:} For each $t \in [T]$, set $w_{t, \l} \leftarrow (\phimt{t, l}{m} B)^\dagger \ymt{t, l}{m}$ and set $\theta_{t, \l} \leftarrow B w_{t, \l}$
            \State {\bfseries Gradient w.r.t $B$:} With $\ymt{t}{m} \equiv \ymt{t, L + \l}{m}$, $\phimt{t}{m} \equiv \phimt{t, L + \l}{m}$, find $\nabla_B f(B, W_{\l}) = \sum_{t=1}^T {\phimt{t}{m}}^\top (\phimt{t}{m} B w_{t, \l} - \ymt{t}{m}) w_{t, \l}^\top$
            \State {\bfseries GD step:} Set $\widehat{B}^+ \leftarrow B - \frac{\gamma}{\pG_m - \pG_{m-1}} \nabla_B f(B, W_{\l})$
            \State {\bfseries Projection step:} Compute $\widehat{B}^+ \overset{QR}{=} B^+ R^+$, $B_{\l} \leftarrow B^+$
        \EndFor
        \State Set $\Bm{m} \leftarrow B_{L}$ and set $\Wm{m} \leftarrow W_{L}$
        \State For each task $t \in [T]$: let $\widehat{\theta}_t^{\sm} = \Bm{m} \widehat{w}_t^{\sm}$\label{line:est}
    \EndFor
\end{algorithmic}
\end{algorithm}

where $\alpha = \frac{\tilde{C}}{\pG_1 T} \sum_{n=1, t=1}^{\pG_1, T} y_\nt^2$, $\tilde{C} = 9 \kappa^2 \mu^2$, and $y_{t, trunc}(\alpha) := \ymt{t}{1} \circ \indic_{\{|\ymt{t}{1}| \leqslant \sqrt{\alpha}\}}$. We extract the top $r$ left singular vectors from $\widehat{\Theta}_{0}$ to derive the initial estimation of $\wB^{\szero}$. In Proposition~\ref{p6}, we provide the initialization guarantee of this approach.

After the initialization step, we alternate between gradient descent (GD) and minimization steps iteratively to estimate $\wB$ and $\wW$. Since each $w_t$ is decoupled in $f_m(\wB, \wW)$, the minimization step for each task is also decoupled as $w_t \in \argmin_{\widetilde{w}_t} \|\ymt{t}{m} - \phimt{t}{m} B \widetilde{w}_t\|^2$. Each task $t$ update $w_t$ as 
\begin{equation*}
w_t = (\phimt{t}{m} B)^\dagger \ymt{t}{m}, 
\end{equation*}
and $B$ via a GD step as $\widehat{B}^+ \leftarrow B - \gamma \nabla_B f(B, W)$, followed by QR decomposition $\widehat{B}^+ \overset{QR}{=} B^+ R^+$ to ensure the orthogonality and stability of the updates. Thus we have $B=B^+.$
%
Beginning with epoch-2, for each epoch $m\geqslant 2$, we utilize the estimation $\thetahatt^{\sm}$ derived from the previous epoch for greedy exploration. 
At the end of the epochs, we obtain the estimate of $(B^\star, W^\star)$.
%

{\cblue {\noindent\bf Practical setting of parameters.} In our algorithm, we set the parameters, GD step size $\gamma$ and the multiplier $\widetilde{C}$ in the spectral initialization step. Our theorem states that $\gamma = \frac{c}{(1 + \frac{0.04}{(1 - 2 \rho)^2})^2 \sigma_{\max}^{\star^2}}$ with $c\leqslant0.5$. However, $\sigma_{\max}^{\star}$ is unknown. Given the initialization matrix $\widehat{\Theta}_0$ provides an approximation to $\Theta^\star$, we set $\sigma_{\max}^\star\approx\|\widehat{\Theta}_0\|$ and $\gamma=\frac{c}{(1 + \frac{0.04}{(1 - 2 \rho)^2})^2 \|\widehat{\Theta}_0\|^2}$. Our analysis requires $\widetilde{C}=9\kappa^2\mu^2$, with $\kappa$ and $\mu$ being functions of $\Theta^\star$ and hence unknown. Using the incoherence assumption, we can set $\kappa^2\mu^2$ by an estimate of its lower bound, $T\max_t\frac{\|\widehat{\theta}_t\|^2}{\|\widehat{\Theta}\|_F^2}$, with $\|\widehat{\theta}_t\|^2=\frac{1}{\pG_1}\sum_{n=1}^{\pG_1}y_{n,t}^2$ and $\|\widehat{\Theta}\|_F^2=\frac{1}{\pG_1}\sum_{t=1}^T\sum_{n=1}^{\pG_1}y_{n,t}^2$.}

\section{Main Results and Guarantees}
This section establishes safety guarantees, presents bounds on the estimation error between the learned parameter $B^+$ and the true parameter $B^\star$, and finally derives the cumulative regret bound.
Let $\epsilon$ be the final desired error between the true parameter $B^\star$ and its estimate.
To establish the safety guarantee, we consider two cases: (a)~the initial epoch ($m=1$) and (b)~subsequent epochs ($m\geqslant 2$). In Lemma~\ref{l8}, we show that the actions taken during the safe exploration phase comply with the safety constraints. Additionally, Theorem~\ref{l9} proves that the actions derived using the estimate $\thetahatt^{\sm}$ for $m \geqslant 2$ also satisfy the safety requirements. 
{\cblue In the subsequent theorem statements, the quantities $\pG_1T$ and $(\pG_m - \pG_{m-1})T$ denote the total number of samples collected across all tasks in the first and subsequent epochs, respectively, while $\pG_1$ and $\pG_m - \pG_{m-1}$ denote the number of samples available per task.}

\begin{theorem}\label{l9}
Set the GD step-size and conservative exploration parameters as $\gamma = \frac{c}{(1 + \frac{0.04}{(1 - 2 \rho)^2})^2 \sigma_{\max}^{\star^2}}$ and $\rho \in [0, 0.25] \cup [0.75, 1]$, and number of GD iterations at least $L\geqslant C\kappa^2\log\Big(\frac{\kbnt+\alpha\rbnt}{2\big(1+\frac{2}{(1-2\rho)^2}\big)\mu\sigma_{\max}^\star\sqrt{\frac{2r}{T}\log\frac{NT}{\delta}}}\Big)$.  Define $\NSR := \frac{T \sigma_\eta^2}{\sigma_{\min}^{\star^2}}$, and let $\epsilon$ represent the desired final error. If 
\begin{eqnarray*}
\pG_1 T &>& C \mu^4 d r \kappa^6 \max(r \kappa^2, \NSR), \\
\pG_1 &>& C (r + \log T + \log d) \max(1, \frac{\NSR}{\mu^2 r \kappa^2}), 
\end{eqnarray*}

and for $m \geqslant 2$, 
\begin{eqnarray*}
\scalemath{1}{(\pG_m - \pG_{m-1}) T} \hspace{-3 mm}&\geqslant &\hspace{-3 mm} \scalemath{1}{C \mu^2 \kappa^4 (d + r) r \max(1, \frac{\kappa^2 \NSR}{\epsilon^2})}, \\
\scalemath{1}{\pG_m - \pG_{m-1}} \hspace{-3 mm}&\geqslant &\hspace{-3 mm}\scalemath{1}{ C (r + \log T + \log d) \max(1, \frac{\NSR}{\epsilon^2 r}),} 
\end{eqnarray*}
then with high probability, the chosen action $x_\nt$ in each epoch $m \geqslant 2$ consistently meets the performance constraint. 
\end{theorem}
{\cblue \noindent{\bf Discussion on Theorem~\ref{l9}.} Theorem~\ref{l9} guarantees safety for all greedy epochs. The conditions define two types of data requirements: $\pG_1T$ and $(\pG_m - \pG_{m-1})T$ represent the total number of samples across the $T$ tasks, per epoch while $\pG_1$ and $\pG_m - \pG_{m-1}$ denote the per-task number sample required in each epoch. 
Specifically, treating $\mu$ and $\kappa$ as problem-dependent constants and hiding logarithmic factors, the first epoch requires $\widetilde{\Omega}(dr\max\{r,\NSR\})$ total samples, whereas each subsequent epoch requires $\widetilde{\Omega}((d+r)r\max\{1,\NSR/\epsilon^2\})$ total samples. 
Under these conditions, the estimation error is sufficiently small, ensuring that the greedy action chosen using $\widehat{\theta}_t^{\sm}$ meets the stage-wise performance constraint with high probability.
Thus, Lemma~\ref{l8} (for safe exploration) and Theorem~\ref{l9} together provide the safety guarantee for the proposed algorithm for all epochs.}

{\cblue \noindent{\bf Proof Sketch (Details in Appendix~\ref{proof_l9}).} By adding and subtracting $x_\nt^\top \widehat{\theta}_t^{\sm}$, safety is guaranteed if
\begin{eqnarray*}
x_\nt^\top \widehat{\theta}_t^{\sm} + x_\nt^\top (\theta_t^\star - \widehat{\theta}_t^{\sm}) \geqslant (1 - \alpha) \rbnt.
\end{eqnarray*}
Using the greedy choice of $x_\nt$, we have
\begin{eqnarray*}
x_\nt^\top \widehat{\theta}_t^{\sm} \geqslant x_{\nt}^{\star^\top} \widehat{\theta}_t^{\sm} = x_{\nt}^{\star^\top} \theta_t^\star + x_{\nt}^{\star^\top} (\widehat{\theta}_t^{\sm} - \theta_t^\star). 
\end{eqnarray*}
Hence, it suffices to show that
\begin{eqnarray*}
x_{\nt}^{\star^\top} \theta_t^\star + x_{\nt}^{\star^\top} (\widehat{\theta}_t^{\sm} - \theta_t^\star) + x_\nt^\top (\theta_t^\star - \widehat{\theta}_t^{\sm}) \geqslant (1 - \alpha) \rbnt. 
\end{eqnarray*}
Utilizing the Gaussian concentration bound uniformly across all task-round pairs yields
\begin{equation*}
x_\nt^\top (\theta_t^\star - \widehat{\theta}_t^{\sm}) \leqslant \sqrt{2 \log\frac{(\pG_{m+1} - \pG_m) T}{\delta}} \|\thetats - \thetahatt^{\sm}\|, ~\textrm{and}
\end{equation*}
\begin{equation*}
x_{\nt}^{\star^\top} (\widehat{\theta}_t^{\sm} - \theta_t^\star) \leqslant \sqrt{2 \log\frac{(\pG_{m+1} - \pG_m) T}{\delta}} \|\thetats - \thetahatt^{\sm}\|. 
\end{equation*}
Thus, safety is guaranteed whenever
\begin{eqnarray*}
2 \sqrt{2 \log\frac{(\pG_{m+1} - \pG_m) T}{\delta}} \|\thetats - \thetahatt^{\sm}\| \leqslant \kbnt + \alpha \rbnt.
\end{eqnarray*}
Given that the estimation error does not increase following the first epoch, Lemma~\ref{l2} together with Theorem~\ref{T5} yields
\begin{eqnarray*}
\|\thetats - \thetahatt^{\sm}\| \leqslant (1 + \frac{2}{(1 - 2\rho)^2}) \mu \sqrt{\frac{r}{T}} (1 - \frac{c}{\kappa^2})^L \sigma_{\max}^\star. 
\end{eqnarray*}
Picking $L$ to ensure that the right-hand side of the above equation satisfies the prior inequality provides the stated condition. The first epoch is guaranteed to be safe by Lemma~\ref{l8}.} 

{\cblue \begin{remark}
The safety requirement for each epoch $m \geqslant 2$ is determined by Theorem~\ref{l9}. 
{\cred The interval $\rho \in [0, 0.25] \cup [0.75, 1]$ arises from the first epoch's estimation and contraction analysis. Specifically, it ensures the Grammian matrix $M$ in the first epoch is well-conditioned. Proposition~\ref{p1} demonstrates that $\|M^{-1}\| \leqslant \frac{2}{(1 - 2 \rho)^2 \pG_1}$, and this factor influences the coefficient estimation and contraction bounds utilized in Lemma~\ref{l2} and Theorem~\ref{T5}. As $\rho\rightarrow\frac{1}{2}$, the constants appearing in the proof become unbounded. Thus, to show the contraction guarantee with fixed numerical constants, we set $(1-2\rho)^2 \geqslant \frac{1}{4}$, which is equivalent to $\rho\in[0,0.25]\cup[0.75,1]$.} 
\end{remark}}


To establish the estimation guarantee, we consider three steps: (i)~spectral initialization, (ii)~initial epoch ($m = 1$), and (iii)~subsequent epochs ($m \geqslant 2$).
The initialization guarantee can be directly obtained from Proposition~\ref{p6} (Theorem~2.2 in \cite{singh2024noisy}). In Theorem~\ref{T5}, we establish a guarantee for exponential error decay during the first epoch. In Theorem~\ref{T11} we prove the exponential error decay for later epochs. Thus, Proposition~\ref{p6} and Theorems~\ref{T5} and~\ref{T11} together prove that the GDmin iterations result in exponential error decay of the reward parameter.
\begin{theorem}\label{T5}
Assume $\SE(B_0, B^\star) \leqslant \delta_0$. If $\delta_0 \leqslant \frac{0.02}{\mu \sqrt{r} \kappa^2}$, $\gamma \leqslant \frac{1}{(1 + \frac{0.04}{(1 - 2 \rho)^2})^2 \sigma_{\max}^{\star^2}}$, and $\rho \in [0, 0.25] \cup [0.75, 1]$,
and if for each GD iteration $\ell$ in epoch $m = 1$, 
\begin{eqnarray*}
\pG_1 T &>& C (d + r) \kappa^2 \max(\mu^2 r \kappa^2, \NSR) \\
\pG_1 &>& C (r + \log T + \log d) \max(1, \frac{\NSR}{\mu^2 r \kappa^2}), 
\end{eqnarray*}
then with probability at least $1 - \ell d^{-10}$, 
$$
\SE(B^+, B^\star) \leqslant \delto:=(1 - 0.23 \gamma \sigma_{\min}^{\star^2})^{\ell} \delta_0. 
$$
\end{theorem}
{\cblue \noindent{\bf Discussion on Theorem~\ref{T5}.} Theorem~\ref{T5} establishes the estimation guarantee for the first conservative epoch. After achieving a sufficiently accurate subspace estimate through spectral initialization, each iteration reduces the subspace error by the factor $1-0.23\gamma\sigma_{\min}^{\star^2}$. Thus, the first epoch not only ensures safety via conservative exploration but also provides an exponentially improving estimate of the shared representation.}

{\cblue \noindent{\bf Proof Sketch (Details in Appendix~\ref{proof_T5}).} Define $P:=I-B^\star{B^{\star}}^\top$. According to Theorem~5.2 in \cite{OurICML}, we have
\begin{align*}
\|P \wB^+\| &\leqslant \|P B\| \|I - \gamma W W^\top\| + \frac{\gamma}{\pG_1} \|\bE[\GradB] - \GradB\|, \\
\SE(B^+, B^\star) &\leqslant \frac{\|P \widehat{B}^+\|}{1 - \frac{\gamma}{\pG_1} \|\bE[\GradB]\| - \frac{\gamma}{\pG_1} \|\GradB - \bE[\GradB]\|}.
\end{align*}
Thus, the proof requires bounding $\|I - \gamma W W^\top\|$, $\|\bE[\GradB] - \GradB\|$, and $\|\bE[\GradB]\|$. 
Lemma~\ref{l2} establishes bounds on $\sigma_{\min}(W)$ and $\sigma_{\max}(W)$ using Bernstein inequality. The upper bound on $\sigma_{\max}(W)$ guarantees that $I - \gamma W W^\top \succeq 0$ under the stated step-size condition, while the lower bound on $\sigma_{\min}(W)$ provides the required contraction bound for $\|I - \gamma W W^\top\|$. 
Lemma~\ref{l4} bounds both $\|\bE[\GradB] - \GradB\|$ and $\|\bE[\GradB]\|$ utilizing the Bernstein inequality. Applying these bounds into the previous inequalities results in the one-step recursion 
\begin{align*} 
\SE(B^+, B^\star) \leqslant \big(1 - 0.23 \gamma \sigma_{\min}^{\star^2} \big)\delt. 
\end{align*} 
Applying this recursion throughout the gradient descent steps in the first epoch provides the stated exponential error decay.}
\begin{theorem} \label{T11}
Assume $\SE(\Bm{1}, B^\star) \leqslant \delta_0 = \frac{c}{\mu \sqrt{r} \kappa^2}$, $\gamma \leqslant \frac{0.5}{\sigma_{\max}^{\star^2}}$. Define $\NSR := \frac{T \sigma_\eta^2}{\sigma_{\min}^{\star^2}}$, and let $\epsilon$ represent the desired final error. If at each GD iteration $\ell$, for any epoch $m \geqslant 2$, 
\begin{eqnarray*}
\scalemath{0.99}{(\pG_m - \pG_{m-1}) T} \hspace{-3 mm}&\geqslant &\hspace{-3 mm} \scalemath{0.99}{C \mu^2 \kappa^4 (d + r) r \max(1, \frac{\kappa^2 \NSR}{\epsilon^2})}, \\
\scalemath{0.99}{\pG_m - \pG_{m-1}} \hspace{-3 mm}&\geqslant &\hspace{-3 mm}\scalemath{0.99}{ C (r + \log T + \log d) \max(1, \frac{\NSR}{\epsilon^2 r}),} 
\end{eqnarray*}
then in the $(\ell + 1)^{\rm th}$ GD iteration, with probability at least $1 - ((m - 1) L + \ell) d^{-10}$, we have
$$
\SE(B^+, B^\star) \leqslant \delto := (1 - 0.36 \gamma \sigma_{\min}^{\star^2})^{((m - 2) L + \ell)} \delta_0. 
$$
\end{theorem}
{\cblue \noindent{\bf Discussion on Theorem~\ref{T11}.} Theorem~\ref{T11} extends the exponential decay guarantee from the first conservative epoch to all following greedy epochs. The target accuracy $\epsilon$ is included in the sample-size conditions because the stochastic estimation error in subsequent epochs must be maintained below the desired final error level. Treating $\mu$ and $\kappa$ as constants and hiding logarithmic factors, each subsequent epoch needs $\widetilde{\Omega}((d+r)r\max\{1,\NSR/\epsilon^2\})$ collect samples, with a per-task sample requirement that scales logarithmically with respect to $T$ and $d$. This demonstrates the advantage of the rank-$r$ shared representation: when $r \ll \min\{d,T\}$, the required total number of sample can be significantly less than the $dT$ samples required to independently estimate all task parameters.}

Proof is given in Section~\ref{proof_T11}. From Theorem~\ref{T5} and Theorem~\ref{T11}, the proposed approach achieves exponential error
decay in all epochs.

\noindent{\bf Sample Complexity.} To analyze the sample complexity, we focus on the term $(\pG_m - \pG_{m - 1}) T$. 
Let $\kappa$ and $\mu$ be constant values. For the first epoch $m = 1$, the sample complexity is $O((d + r) \max\{r, \NSR\})$. For any epoch $m \geqslant 2$, the required sample complexity is $O((d + r) r \max\{1, \frac{\NSR}{\epsilon^2}\})$. 
Since rank $r \geqslant 1$ and the desired final error $\epsilon \ll 1$, the sample complexity of our proposed algorithm in each epoch is $O((d + r) r \max\{1, \frac{\NSR}{\epsilon^2}\})$. Without making the low-rank assumption and without using our algorithm, if we were to perform matrix inversion for $\phimt{t}{m}$ in order to extract each vector $\thetats$ from $\ymt{t}{m}$, we would need at least $(\pG_m - \pG_{m-1}) T \geqslant T d$ samples per epoch. 

\noindent{\bf Time and Communication Complexity.} 
To achieve a $\epsilon$-error between the estimate $\wB$ and the true parameter $B^\star$, the time complexity is $O((\pG_m - \pG_{m-1}) T d r \log{\frac{1}{\epsilon}})$, while the communication complexity is $O(d r \kappa \log{\frac{1}{\epsilon}})$. 

We provide the guarantee for the cumulative regret bound. 
\begin{theorem}\label{T10}
Set $\gamma \leqslant \frac{c}{(1 + \frac{0.04}{(1 - 2 \rho)^2})^2 \sigma_{\max}^{\star^2}}$, {\cblue $\rho \in [0, 0.25] \cup [0.75, 1] \cap \Big[0, \frac{\alpha\rbnt}{\rbnt+\sqrt{2\frac{r}{T}\log\frac{\pG_1 T}{\delta}}\mu\sigma_{\max}^\star}\Big]$}, and $L = C \kappa^2 \log \Big(\frac{\sqrt{2 (\rho^2 - \rho + 1)}}{1 + \frac{2}{(1 - 2 \rho)^2}}\Big)$. Let $\epsilon$ represent the desired final error. If
\begin{eqnarray*}
\pG_1 T &>& C \mu^4 d r \kappa^6 \max(r \kappa^2, \NSR), \\
\pG_1 &>& C (r + \log T + \log d) \max(1, \frac{\NSR}{\mu^2 r \kappa^2}), 
\end{eqnarray*}
and for $m \geqslant 2$, 
\begin{eqnarray*}
\scalemath{0.98}{(\pG_m - \pG_{m-1}) T} &\geqslant& \scalemath{0.98}{C \mu^2 \kappa^4 (d + r) r \max(1, \frac{\kappa^2 \NSR}{\epsilon^2})}, \\
\pG_m - \pG_{m-1} &\geqslant & C (r + \log T + \log d) \max(1, \frac{\NSR}{\epsilon^2 r}), 
\end{eqnarray*}
then with probability at least $1 - 2 \delta - M L d^{-10}$, the cumulative regret $\pR_\NT$ is bounded as 
\begin{equation*}
\pR_\NT \leqslant 4 \sqrt{2} \mu \sigma_{\max}^\star \sqrt{r N T (\rho^2 - \rho + 1) \log \frac{1}{\delta}}. 
\end{equation*}
\end{theorem}
{\cblue \noindent{\bf Discussion on Theorem~\ref{T10}.} Theorem~\ref{T10} provides the regret guarantee derived using the safety and estimation results. Given the stated sample size and parameter conditions, while considering $\mu$, $\sigma_{\max}^{\star}$ and $\rho$ as constants, the cumulative regret scales as $\widetilde{O}(\sqrt{rNT})$. Thus, the regret is sublinear with respect to the number of samples $N$, the number of tasks $T$, and the rank $r$ instead of the ambient dimension $d$.} We note that the standard (naive) algorithm for {\cblue conservative} bandits results in a regret bound of $\widetilde{O}(d\sqrt{N})$ for each task \cite{moradipari2020stage, abbasi2011improved}. Thus, solving $T$ tasks separately would result in a regret of $\widetilde{O}(Td\sqrt{N})$. Theorem~\ref{T10} thus validates the efficiency of MTRL. {\cred Prior work on MTRL bandits \cite{hu2021near} assumes access to the optimal solution of the non-convex estimation problem and obtained regret bound that has a linear dependence on $T$. In contrast, even without this assumption, we obtain sublinear regret, where the above bound is influenced by Assumption~\ref{assume:incoherence}.} 

{\cblue \noindent{\bf Proof Sketch (Details in Appendix~\ref{proof_T10}).} Write $R_{N,T} = R_{N,T}^{(1)} + R_{N,T}^{(2)}$. In the first epoch, the action is $x_\nt = (1 - \rho) x_\bnt + \rho \zeta_\nt$. Hence, utilizing Gaussian concentration and the incoherence condition, we have
\begin{eqnarray*} 
R_{N,T}^{(1)} \leqslant 2 \mu \sigma_{\max}^{\star} \sqrt{\pG_1 r T (\rho^2 - \rho + 1) \log\frac{1}{\delta}}. 
\end{eqnarray*}
In subsequent epochs, the greedy selection rule gives that $x_{n,t}^{\top} \widehat{\theta}_{m-1, t} \geqslant x_{n,t}^{\star^\top} \widehat{\theta}_{m-1, t}$. Thus, 
\begin{eqnarray*} 
x_{n,t}^{\star^\top } \theta_t^\star - x_{n,t}^{\top} \theta_t^\star \leqslant (x_{n,t}^\star - x_{n,t})^\top (\theta_t^\star - \widehat{\theta}_{m-1,t}). 
\end{eqnarray*}
By selecting $L$ and applying Lemma~\ref{l2}, we have
\begin{eqnarray*} 
\|\theta_t^\star - \widehat{\theta}_{m-1, t}\| \leqslant \|\theta_t^\star - \widehat{\theta}_{1, t}\| \leqslant \sqrt{2 (\rho^2 - \rho + 1)} \mu \sqrt{\frac{r}{T}} \sigma_{\max}^{\star}. 
\end{eqnarray*}
Applying Gaussian concentration gives
\begin{eqnarray*} 
R_{N,T}^{(2)} \leqslant 2 \mu \sigma_{\max}^{\star} \sqrt{2 (N - \pG_1) r T (\rho^2 - \rho + 1)\log\frac{1}{\delta}}. 
\end{eqnarray*}
Combining the two bounds, we derive
\begin{eqnarray*} 
R_{N,T} \leqslant 4 \sqrt{2} \mu \sigma_{\max}^{\star} \sqrt{r N T (\rho^2 - \rho + 1) \log\frac{1}{\delta}}. 
\end{eqnarray*}
}
%
%
{\cblue \begin{remark}
Theorem~\ref{T10} provides an upper bound on the regret. The parameter $\mu$ is referred to as the incoherence parameter and arises from the column-wise incoherence condition of $W^\star = \Sigma^\star V^\star$, namely
$
\max_t \|w_t^{\star}\|_2 \leqslant \mu \sqrt{\frac{r}{T}} \sigma_{\max}^{\star}
$. Specifically, Assumption~\ref{assume:incoherence} states that $0 < l \leqslant \|w_t^\star\|_2 \leqslant u$ for all $t\in[T]$, and Appendix~\ref{app:prelim} demonstrates that one can select $\mu = \frac{u}{l} \geqslant 1$. Thus, $\mu$ is a column-imbalance parameter, i.e., it does not scale with $d$, $T$, or $r$. Intuitively, $\mu$ measures whether the energy of the task-specific coefficient vectors $w_t^\star$ is evenly distributed across the $T$ columns. When $\mu$ approaches $1$, the column norms are well balanced. When $\mu$ increases, certain columns may become more concentrated. Furthermore, we have
$
\frac{1}{T} \sum_{t=1}^T \|w_t^\star\|_2^2 = \frac{\|W^\star\|_F^2}{T} = \frac{\|\Sigma^\star\|_F^2}{T}. 
$
Consequently, the root-mean-square column norm $\sqrt{\frac{\sum_{i=1}^r \sigma_{i}^{\star^2}}{T}}$ lies between $\sqrt{\frac{r}{T}} \sigma_{\min}^{\star}$ and $\sqrt{\frac{r}{T}} \sigma_{\max}^{\star}$. Therefore, $\sqrt{\frac{r}{T}} \sigma_{\max}^{\star}$ serves as an inherent standard scale for the task-specific coefficient vectors, while $\mu$ only allows the maximum column norm to exceed this scale by a constant factor. Moreover, $\kappa = \frac{\sigma_{\max}^{\star}}{\sigma_{\min}^{\star}}$ represents the condition number of $\Sigma^\star$. It measures the conditioning of the nonzero singular spectrum. When $\kappa$ approaches $1$, all $r$ nonzero singular values show comparable magnitudes, indicating that the latent subspace directions have similar levels of informativeness. A large $\kappa$ indicates that certain singular directions are significantly weaker than others, complicating subspace estimation and recovery. Treating $\kappa$ as a numerical constant aligns with the standard well-conditioned scenario where the nonzero singular values have comparable magnitudes, thus $\kappa = O(1)$ is common in our settings. 
\end{remark}}

\begin{remark}
\cite{kazerouni2017conservative, moradipari2020stage} explored conservative linear bandits under both known and unknown baseline reward settings. Assumption~3 in \cite{moradipari2020stage} imposes a bounded baseline reward condition, $r_l\leqslant\rbnt\leqslant r_h$, which ensures theoretical guarantees even when the exact baseline reward is unknown. Our proposed algorithm, along with its theoretical guarantees, remains valid even in the unknown baseline setting. Specifically, {\cblue during the initial conservative exploration phase, the sufficient condition in Lemma~\ref{l8} can be satisfied by selecting the conservative parameter $\rho$ based on the lower bound $r_l$ and upper bound $r_h$, resulting in a consistently safe selection. In the subsequent greedy epochs, a sufficient condition for the number of gradient descent (GD) iterations $L$  in Theorem~\ref{l9} can be determined using upper bounds $\kappa_h + \alpha r_h$ for the numerator of the logarithmic term, rather than the exact baseline reward.} This modification seamlessly propagates through the analysis, ensuring that our theoretical guarantees remain intact. 
\end{remark}

\section{Simulations}\label{sec:sim}
\begin{figure*}
\vspace{-2 mm}
\subcaptionbox{\label{fig:2}}{\includegraphics[scale=0.345]{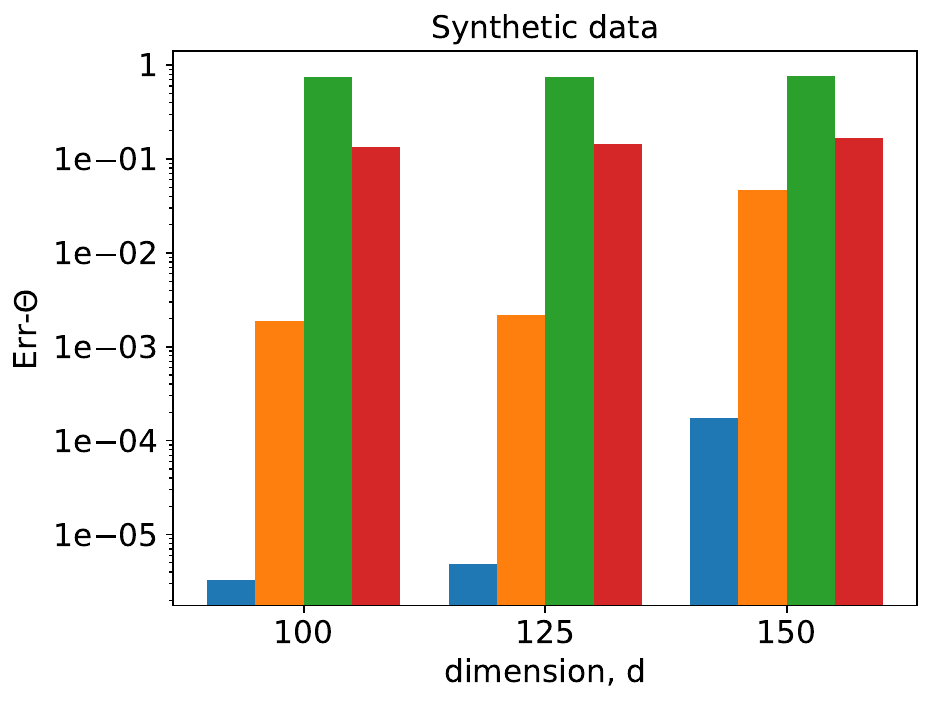}\vspace{-2 mm}}\hspace{0.8 em}%
\subcaptionbox{\label{fig:1}}{\includegraphics[scale=0.345]{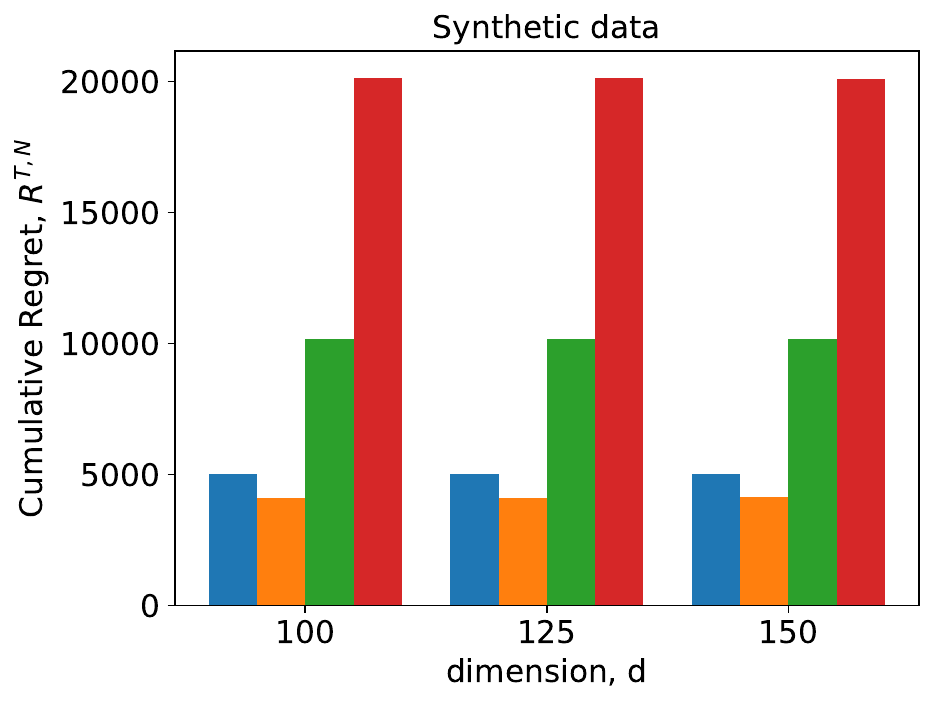}\vspace{-2 mm}}\hspace{0.8 em}%
\subcaptionbox{\label{fig:3}}{\includegraphics[scale=0.345]{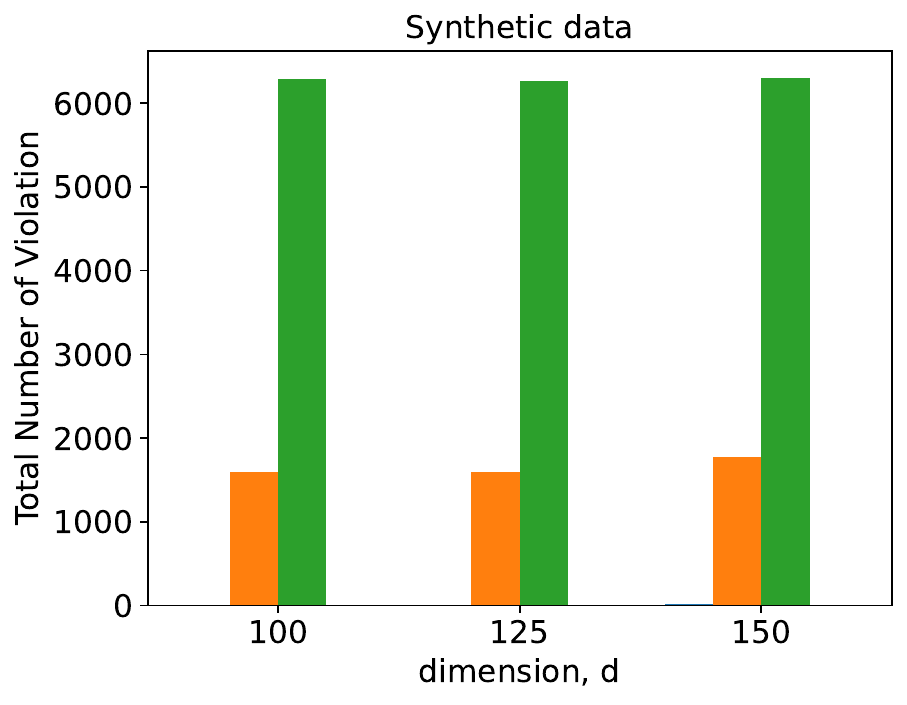}\vspace{-2 mm}}\hspace{0.8 em}
\vspace{-1mm}
\subcaptionbox{\label{fig:5}}{\includegraphics[scale=0.345]{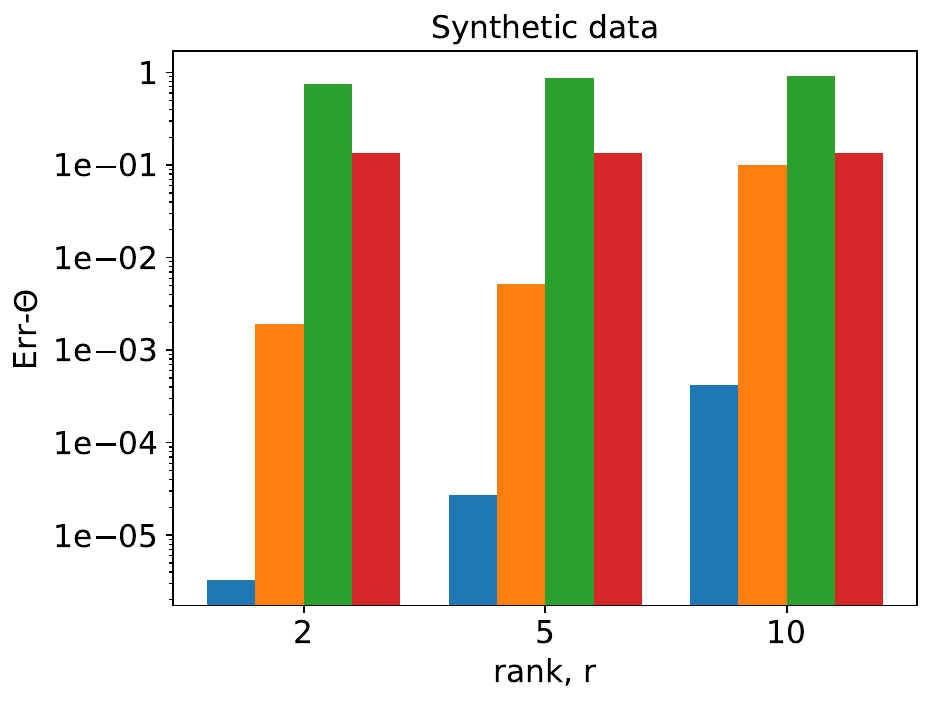}\vspace{-2 mm}}\hspace{0.8 em}%
\subcaptionbox{\label{fig:4}}{\includegraphics[scale=0.345]{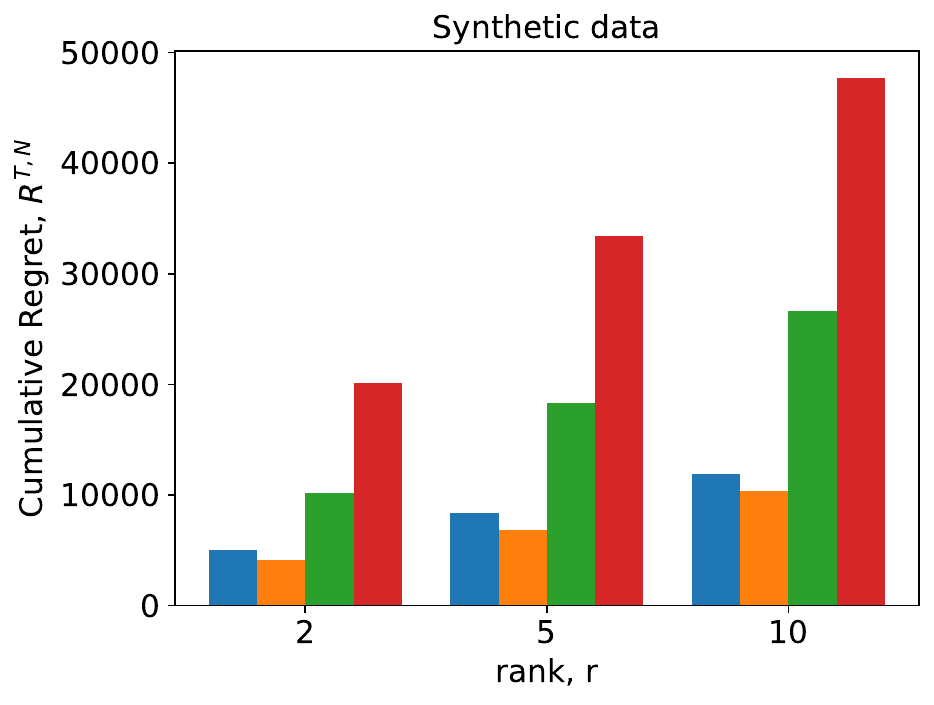}\vspace{-2 mm}}\hspace{0.8 em}%
\subcaptionbox{\label{fig:6}}{\includegraphics[scale=0.345]{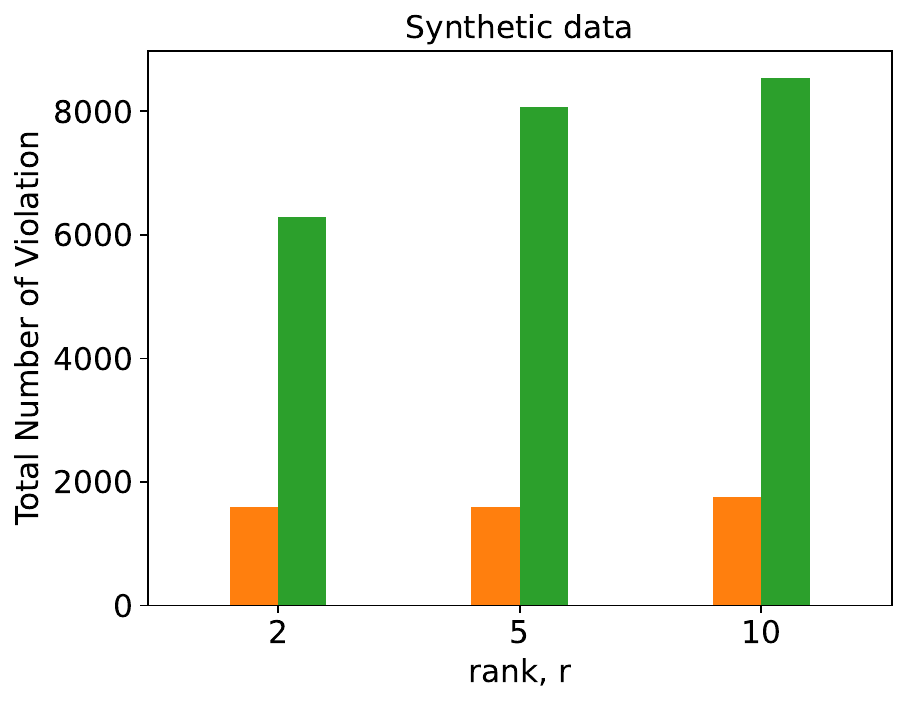}\vspace{-2 mm}}\hspace{0.8 em}
\vspace{-1mm}
\subcaptionbox{\label{fig:8}}{\includegraphics[scale=0.345]{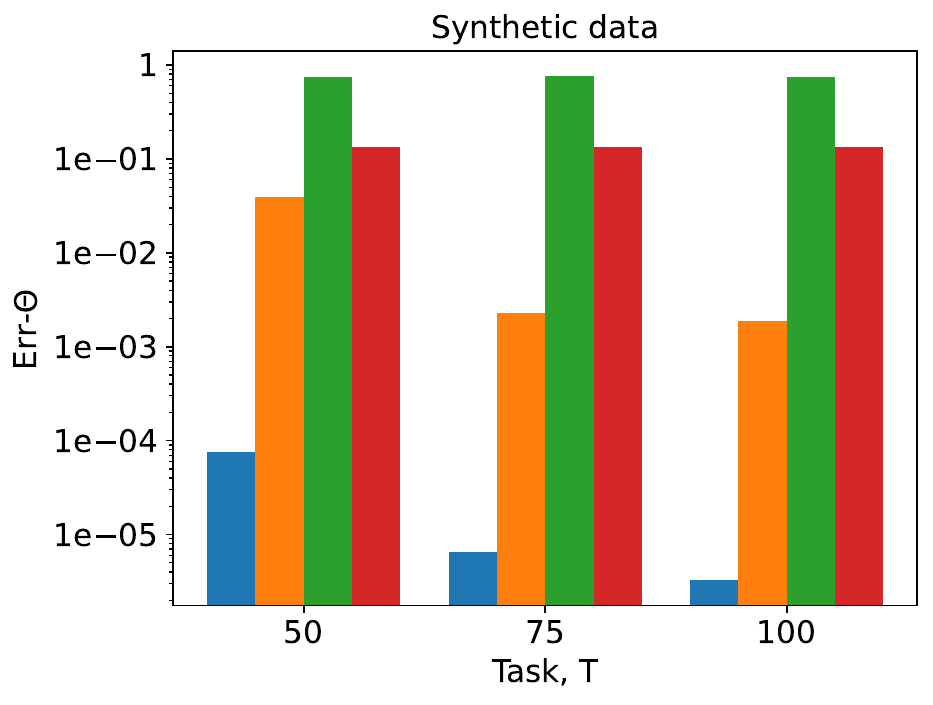}\vspace{-2 mm}}\hspace{0.8 em}%
\subcaptionbox{\label{fig:7}}{\includegraphics[scale=0.345]{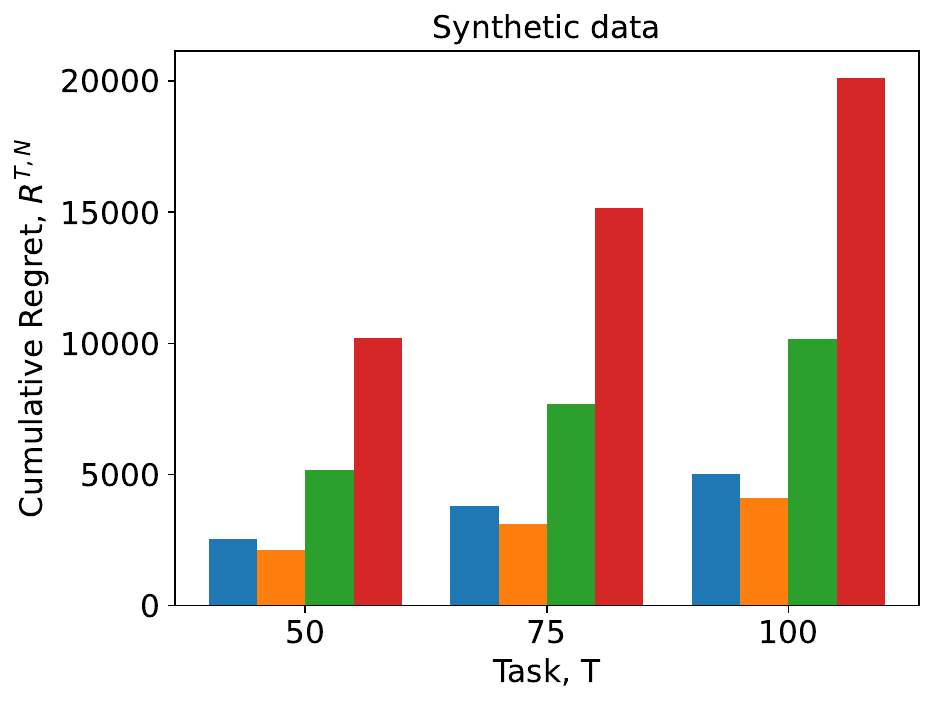}\vspace{-2 mm}}\hspace{0.8 em}%
\subcaptionbox{\label{fig:9}}{\includegraphics[scale=0.345]{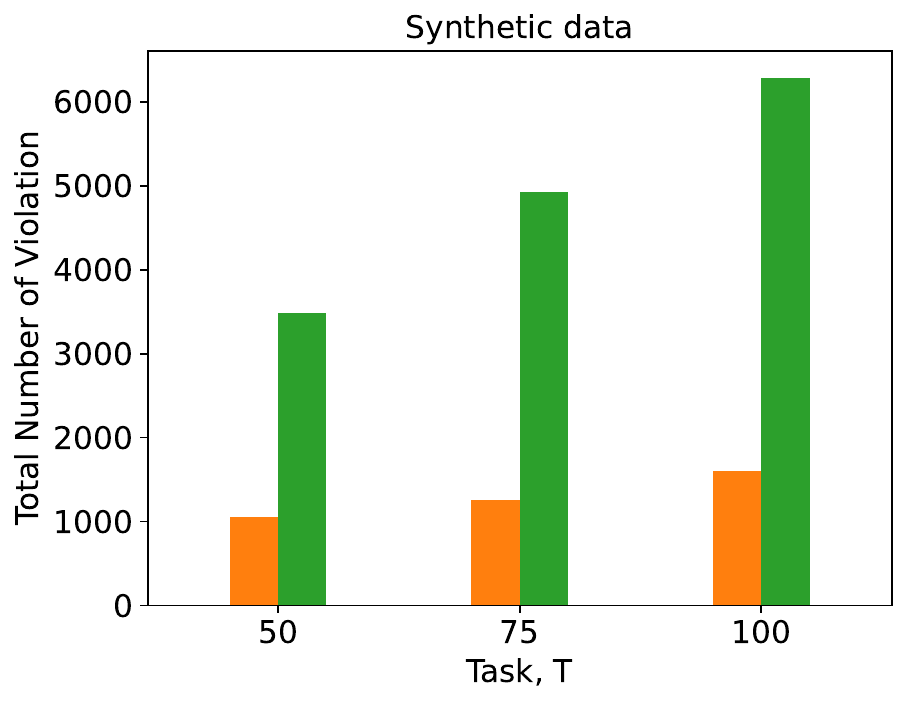}\vspace{-2 mm}}\hspace{0.8 em}
\vspace{-1mm}
\subcaptionbox{\label{fig:10}}{\includegraphics[scale=0.2]{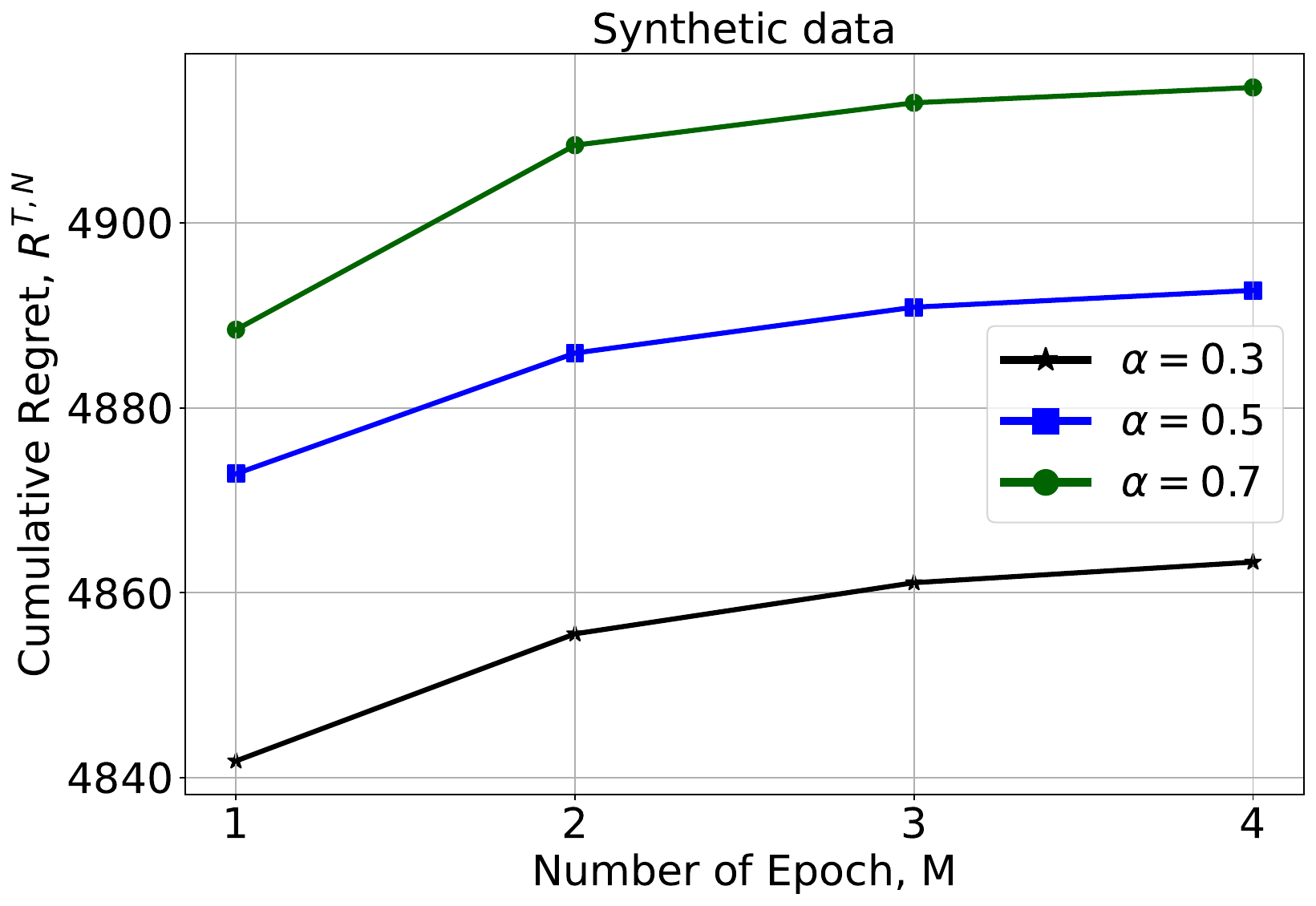}\vspace{-2 mm}}\hspace{0.8 em}%
\subcaptionbox{\label{fig:11}}{\includegraphics[scale=0.2]{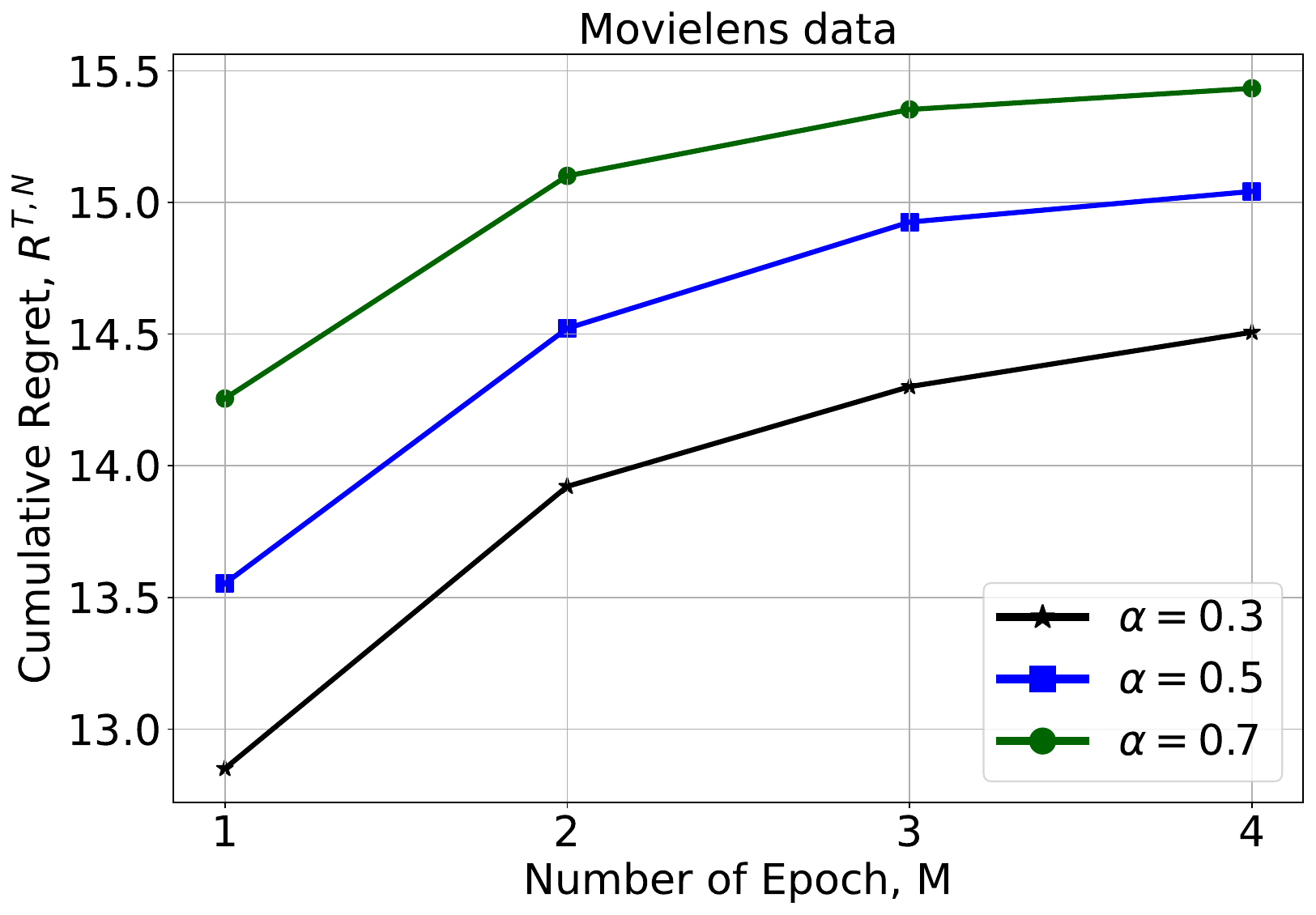}\vspace{-2 mm}}\hspace{0.2 em}%
\subcaptionbox{\label{fig:12}}{\includegraphics[scale=0.36]{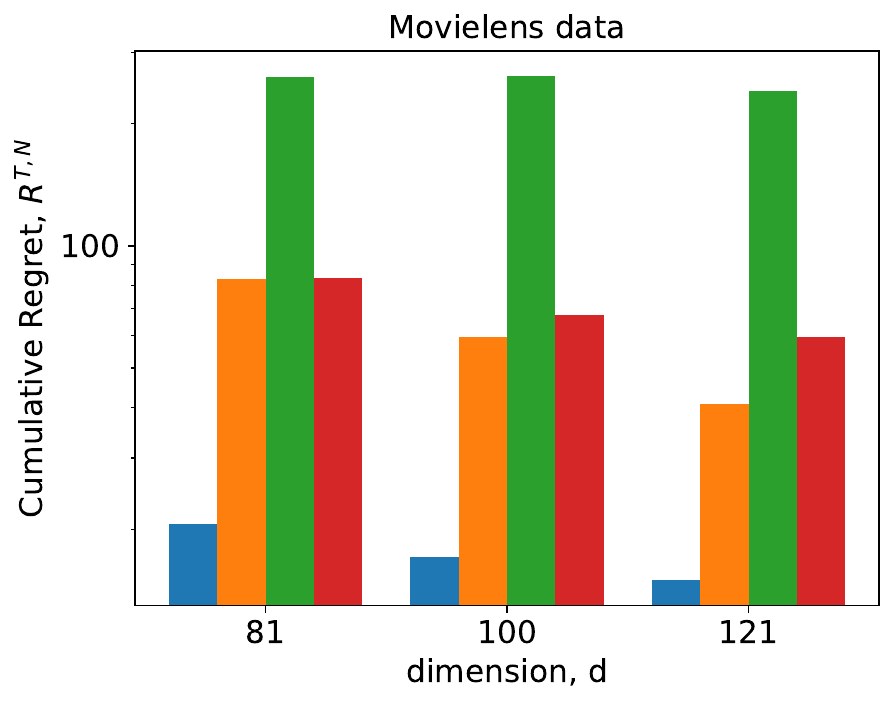}\vspace{-3 mm}}\hspace{0.8 em}
\vspace{-2mm}
\subcaptionbox{\label{fig:13}}{\includegraphics[scale=0.365]{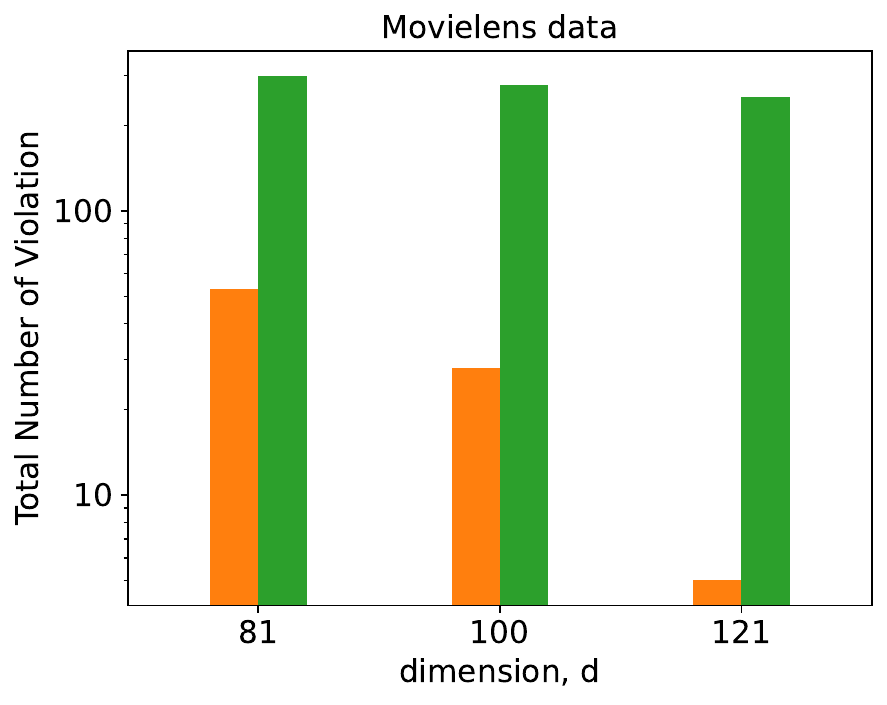}\vspace{-1 mm}}\hspace{0.8 em}%
\subcaptionbox{\label{fig:14}}{\includegraphics[scale=0.365]{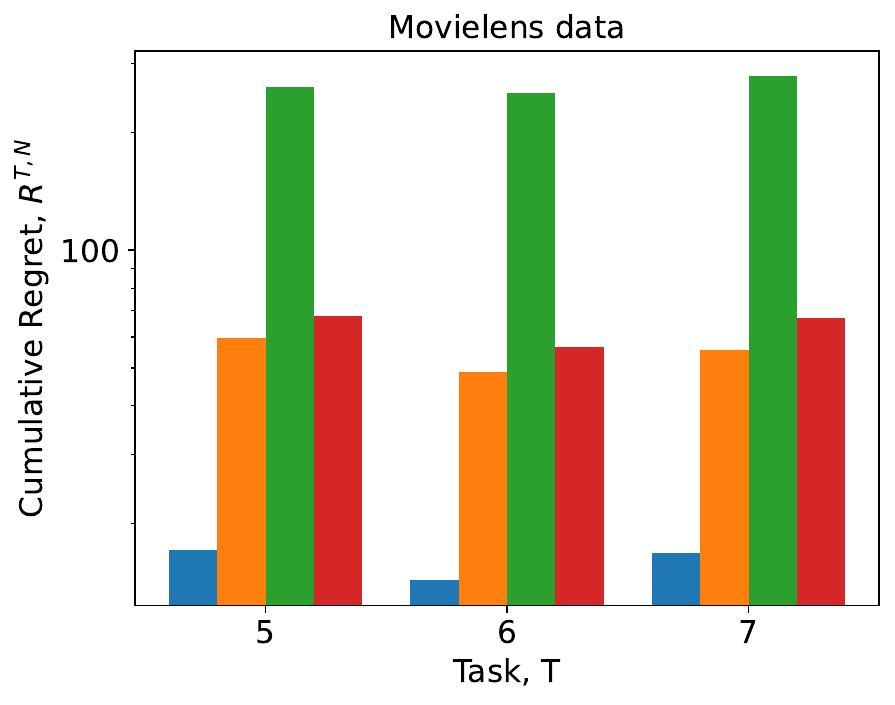}\vspace{-1 mm}}\hspace{0.8 em}%
\subcaptionbox{\label{fig:15}}{\includegraphics[scale=0.365]{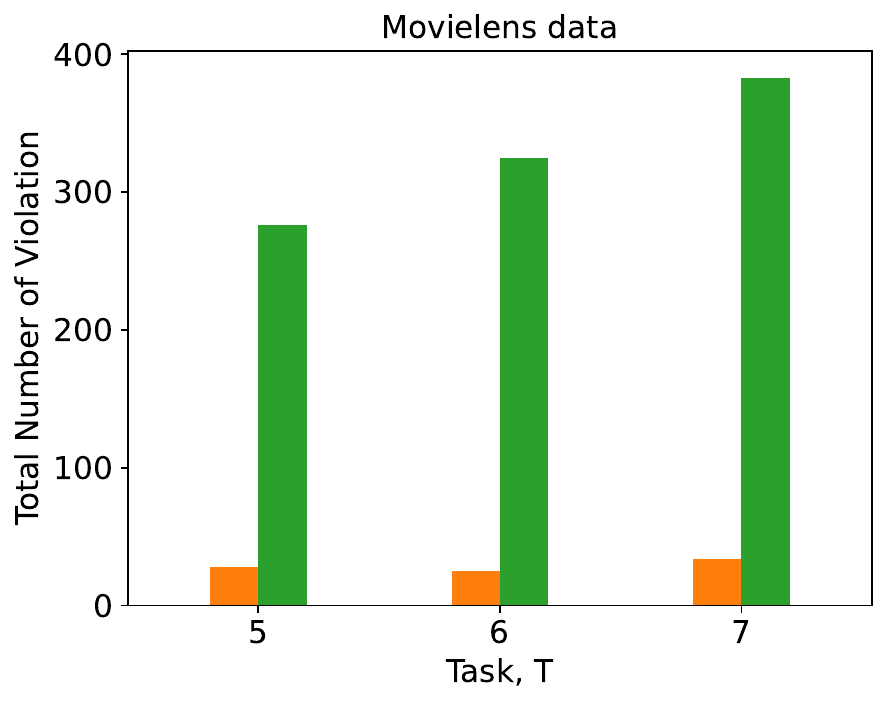}\vspace{-1 mm}}\hspace{0.8 em}
\caption{\footnotesize 
\legendsquare{deepblue}~Our proposed algorithm (Safe-AltGDmin),
  \legendsquare{orange}~Trace-norm (convex relaxation),
  \legendsquare{deepgreen}~Method of Moments (MoM), \legendsquare{deepred}~Thompson sampling (naive baseline).
  The plots present estimation error, cumulative regret, and number of violations for synthetic data (Figures~\ref{fig:2} to ~\ref{fig:10}) and Movielens data (Figures~\ref{fig:11} to ~\ref{fig:15}). The number of epochs is $M = 4$ and $50$ data samples in each epoch {\cblue for all algorithms}.}\label{fig:main1}
\end{figure*}

This section presents an experimental analysis of our proposed {\cblue Safe-AltGDmin} algorithm using both synthetic data and the real-world Movielens dataset. We compared our proposed algorithm with three benchmark algorithms: (i) the trace-norm convex relaxation approach \cite{cella2023multi, du2020few, chen2022active}, (ii) the Thompson Sampling (TS)-based algorithm \cite{moradipari2020stage} for independently solving $T$ constrained tasks, and (iii) the Method-of-Moments (MoM) algorithm from \cite{tripuraneni2021provable,yang2020impact}. The algorithm based on trace norm relaxation utilizes a convex relaxation to transform the rank constraint into a trace norm convex constraint, iteratively estimating $\widehat{\Theta}$ and the regularization parameter $\lambda$. TS algorithm independently estimates $\widehat{\theta}_t$ for each task $t$ during each iteration and constructs a safe action set. The MoM algorithm, during the first epoch, uses top-$r$ SVD to estimate $\wB$ from $\frac{1}{\pG_1 T} \sum_{t=1}^T \sum_{n=1}^{\pG_1} x_{n, t}x_{n, t}^\top y_{n, t}^2$ and in the second epoch utilizes a least-squares estimator to estimate $\wW$; the estimates $\wB$ and $\wW$ are subsequently applied to compute $\widehat{\Theta}$. In later epochs, the algorithm selects actions according to a greedy strategy utilizing $\widehat{\Theta}$. In our experiments, we varied parameters dimension $d$, rank $r$, and the number of tasks $T$. All experiments were carried out in Python, and for synthetic data experiments, the results are averaged over independent $100$ trials.
%
\subsection {Datasets}
\noindent{\bf Synthetic data:} 
In synthetic data experiments, we used the following default setting: dimension $d = 100$, rank $r = 2$, number of tasks $T = 100$, number of actions $K = 10$, and  $5$-th best action is selected as the baseline action. We randomly generated orthonormal matrices for $B^\star$ and $W^\star$ from an i.i.d. Gaussian distribution. The $\Phi_t$ matrices are independently sampled from a standard Gaussian distribution. We used a Gaussian distribution for the reward noise, with mean zero and variance $10^{-6}$. 

\noindent{\bf Movielens data:} 
We evaluated our proposed algorithm using the Movielens-100K dataset from \cite{harper2015movielens}. The dataset was preprocessed using collaborative filtering, resulting in the rating matrix $R \in \bR^{943 \times 1682}$. The ratings range from $0$ to $5$, and we normalized to the range $[0, 1]$. Subsequently, we implemented Non-negative Matrix Factorization (MNF) to decompose the matrix as $R = U H$, with the latent factor dimension at $\sqrt{d}$,  where $U \in \bR^{943 \times \sqrt{d}}$ and $H \in \bR^{\sqrt{d} \times 1682}$. We utilized the $k$-means algorithm to cluster the columns of $H$ into $T$ distinct groups to derive feature vectors $x_\nt$ for various tasks. For each task $t$, the feature vector is obtained by computing the outer product of the $i^{\rm th}$ row of $U$ and $j^{\rm th}$  column of the $t$-th cluster in $H$, subsequently vectorizing the resulting $\sqrt{d} \times \sqrt{d}$ matrix. The reward noise is generated from a Gaussian distribution with a mean of zero and a variance of $10^{-6}$. The reward parameter $\theta^\star_t$ for each task $t$ under this modeling is the vectorized $I_{\sqrt{d}}$ matrix. Thus, $\Thetas = [\theta_1^\star \cdots \theta_T^\star]$ has rank $r = 1$.  We set the default setting of the parameters as $d = 100$, $T = 5$, with the baseline action chosen as the $5$-th optimal action. 
%
\subsection{Evaluation of the Proposed Algorithm}
\begin{figure}[H]
\centering
\includegraphics[width=0.3\textwidth]{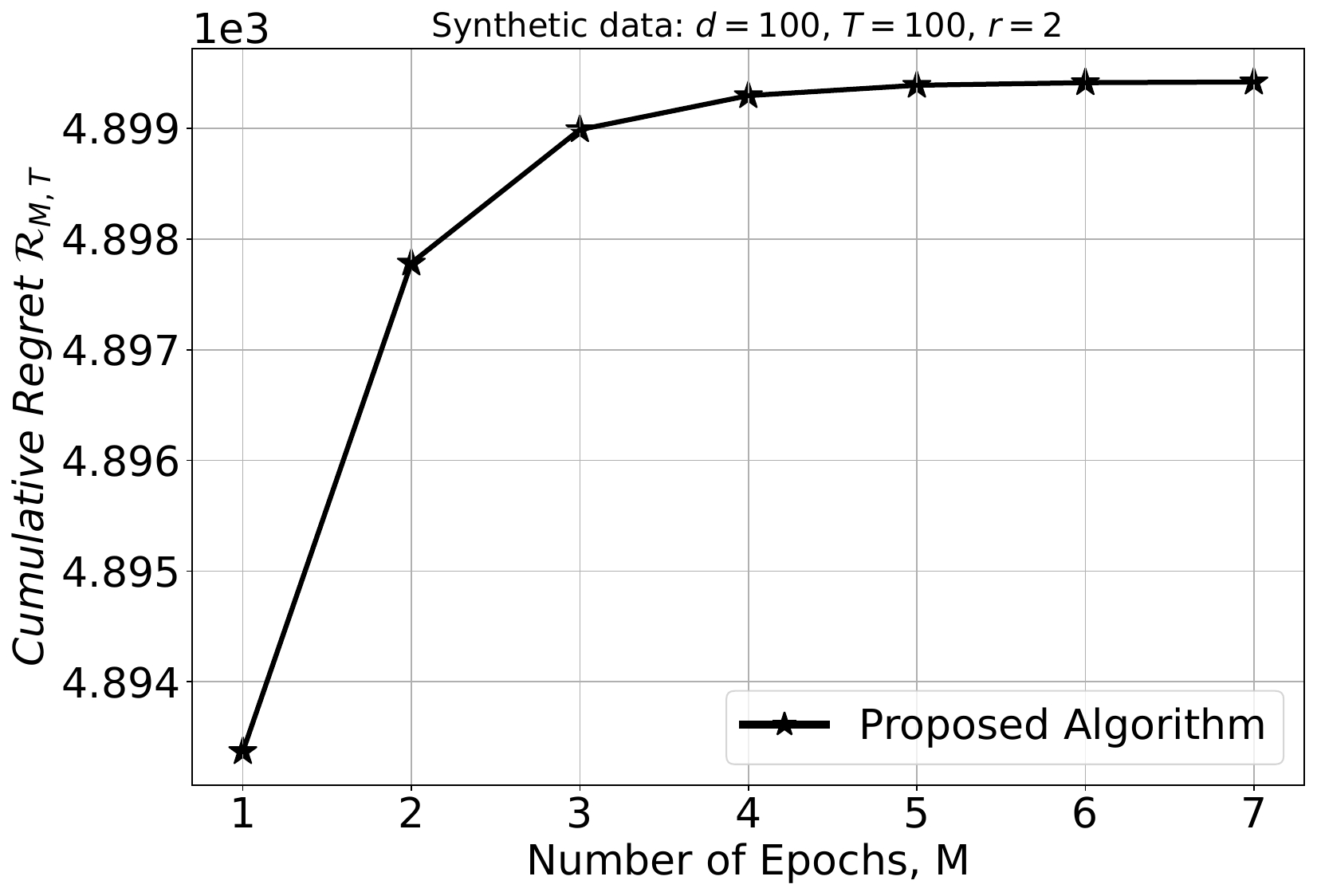}
\caption{Cumulative regret vs. number of epochs.}
\label{proposed_algorithm}
\end{figure}
We start by evaluating the proposed algorithm on synthetic data. 
Figure~\ref{proposed_algorithm} shows a sublinear increase in cumulative regret as the number of epochs increases (each epoch is $50$ rounds), indicating that our algorithm gradually enhances the efficiency of decision making over epochs. In the following subsections, we evaluate the proposed algorithm against various benchmark algorithms. To evaluate performance, we vary $d$, $r$, $T$, and $\alpha$. 

\subsection{Results and Discussions}
\noindent{\bf Estimation error:} 
Figures~\ref{fig:2}, \ref{fig:5}, and~\ref{fig:8} illustrate the estimation error Err-$\Theta =\frac{\|\widehat{\Theta} - \Thetas\|_F}{\|\Thetas\|_F}$ for synthetic data by varying dimensions $d$ and the number of tasks $T$. We compare against the three benchmark algorithms.
 The $y$-axis of the plots is the percentage error in estimating $\widehat{\Theta}$  in the final epoch compared to the true parameter $\Thetas$. The plots clearly show that the performance of our proposed algorithm substantially exceeds that of the benchmark algorithms. As the number of tasks $T$ increases,  learning improves, validating the benefits of MTRL.

\noindent{\bf Cumulative regret and number of violations:} 
We compared the cumulative regret and constraint satisfaction of our proposed algorithm with three benchmark algorithms. Figures~\ref{fig:1},~\ref{fig:4}, and~\ref{fig:7} present the cumulative regret for synthetic data when dimension $d$, rank $r$, and number of tasks $T$ are varied. Figures~\ref{fig:2},~\ref{fig:5}, and~\ref{fig:8} present the number of constraint violations across the algorithms for different values of $d$, $r$, and $T$. 
Both our proposed algorithm and the TS algorithm with safe set estimation from \cite{moradipari2020stage} perform safety-focused exploration, while the trace norm-based and MoM algorithms do not account for the constraints, allowing them to explore without restrictions. This allows these two benchmark algorithms to achieve a lower cumulative regret. However, this leads to a large number of constraint violations. Both our proposed algorithm and the TS algorithm with safe set estimation consistently meet the performance constraints, in contrast to trace-norm-based and MoM approaches, which demonstrate numerous violations. The TS algorithm, however, has a larger cumulative regret. 
As the number of tasks $T$ increases, the cumulative regret of our proposed algorithm grows sublinearly. In contrast, the method from \cite{moradipari2020stage}, which treats each task independently, leads to linear regret. Similarly, as the rank $r$ increases, the cumulative regret of our proposed algorithm also increases sublinearly, as expected. This validates our proposed algorithm. 
Figures~\ref{fig:12} and~\ref{fig:14} present cumulative regret plots by varying $d$ and $T$ for the Movielens dataset. Figures~\ref{fig:13} and~\ref{fig:15} present the number of violations for the Movielens dataset. These plots clearly illustrate that our proposed algorithm consistently outperforms the benchmark algorithms without any violations. 
Furthermore, Figures~\ref{fig:10} and~\ref{fig:11} show the cumulative regret of our proposed algorithm for synthetic and Movielens datasets, respectively, as the constraint parameter $\alpha$ is altered. 
All experiments consistently demonstrate the effectiveness of our proposed approach.

\section{Conclusion}\label{sec:conc}
This paper presented a novel Safe-AltGDmin algorithm to solve the multi-task representation learning problem in {\cblue conservative} linear bandits. In our proposed algorithm, we effectively addressed CMTRL with a low-rank common representation, guaranteeing that all chosen actions meet the performance constraints while concurrently providing an accurate estimation of the shared representation. We provided theoretical guarantees for Safe-AltGDmin, including the cumulative regret bound and compliance with performance constraints. We conducted experiments using both synthetic data and real-world Movielens data to assess the effectiveness of our method. The experimental results showed our algorithm's effectiveness regarding regret and its ability to consistently meet the constraints compared to the benchmark algorithms. 

\vspace{-3 mm}
\bibliographystyle{myIEEEtran}
\bibliography{Bandits}

\appendices
\section{Preliminaries}\label{app:prelim}
\begin{prop}[Theorem~2.8.1, \cite{vershynin2018high}] \label{proposition1}
Let $c > 0$ is an absolute constant, and let $X_1, \cdots, X_N$ be independent, mean zero, sub-exponential random variables. Then, for every $g \geqslant 0$, we have
$$
\mathbb{P} \Bigl\{ \big|\sum_{i=1}^N X_i\big| \geqslant g \Bigl\} \leqslant 2 \exp \Big[- c \min \big(\frac{g^2}{\sum_{i=1}^N \norm{X_i}_{\psi_1}^2}, \frac{g}{\max_i \norm{X_i}_{\psi_1}}\big)\Big].
$$
\end{prop}

\begin{prop}[Chernoff bound for Gaussian] \label{proposition2}
Let $X \sim \N(\mu_x, \sigma_x^2)$, then for all $g > 0$, 
$
\mathbb{P} \Bigl\{X - \mu_x \geqslant g \Bigl\} \leqslant \exp(-\frac{g^2}{2 \sigma_x^2}). 
$
\end{prop}

\begin{prop}[Theorem~2.2, \cite{singh2024noisy}] \label{p6}
Pick a $\delta_0 < 0.1$. Define the noise-to-signal ratio as $\NSR := \frac{T \sigma_\eta^2}{\sigma_{\min}^{\star^2}}$. If $\pG_1 T > C \mu^2 \kappa^2 (d r \frac{\kappa^2}{\delta_0^2} + \frac{d}{\delta_0^2} \NSR)$, then with probability at least $1 - \exp(- c (d + T))$, 
\begin{align*}
\SE(B_0, B^\star) \leqslant \delta_0. 
\end{align*} 
\end{prop}
\begin{defi}(Incoherence)\label{define:incoherence}
A rank-$r$ matrix $M\in\bR^{d_1\times d_2}$ is defined as $\mu$-column-wise incoherent if for every column $m_i\in\bR^{d_1}$ of $M$,  $\max_{i\in[d_2]}\|m_i\|_2\leqslant\mu\sqrt{\frac{d_1}{d_2}}\|M\|_2$, where $\mu\geqslant1$ is a constant that  remains invariant with respect to $d_1$, $d_2$, $r$.
\end{defi}
Under Assumption~\ref{assume:incoherence}, we derive that
\begin{align*}
\|W^\star\|_F=\sqrt{\sum_{t=1}^T\|w_t^\star\|_2^2}\geqslant\sqrt{T}l.
\end{align*}
Furthermore, the Frobenius norm of $W^\star$ satisfies
\begin{align*}
\|W^\star\|_F=\sqrt{\sum_{i=1}^r\sigma_i^2(W^\star)}\leqslant\sqrt{r}\sigma_{\max}^\star.
\end{align*}
By defining $\mu=\frac{u}{l}\geqslant1$, we conclude that
\begin{align*}
\|w_t^\star\|_2\leqslant u=\frac{u}{l}\frac{\sqrt{T}l}{\sqrt{T}}\leqslant\frac{u}{l}\sqrt{\frac{r}{T}}\sigma_{\max}^\star=\mu\sqrt{\frac{r}{T}}\sigma_{\max}^\star.
\end{align*}
According to Definition~\ref{define:incoherence}, $W^\star$ is $\mu$-column-wise incoherent.

Define $P:=I-B^\star{B^{\star}}^\top$, $G:=B^\top \Thetas$, $g_t := B^\top \thetats$ for $t \in [T]$,
\begin{align*}
\GradB_1&=\sum_{t=1}^T \sum_{n=0}^{\pG_1} x_{b_\nt} x_{b_\nt}^\top (\theta_t - \thetats) w_t^\top \\
\GradB_2&=\sum_{t=1}^T \sum_{n=0}^{\pG_1} x_{b_\nt} \zeta_\nt^\top (\theta_t - \thetats) w_t^\top \\
\GradB_3&=\sum_{t=1}^T \sum_{n=0}^{\pG_1} \zeta_\nt x_{b_\nt}^\top (\theta_t - \thetats) w_t^\top \\
\GradB_4&=\sum_{t=1}^T \sum_{n=0}^{\pG_1} \zeta_\nt \zeta_\nt^\top (\theta_t - \thetats) w_t^\top
\end{align*}
\begin{align*}
&\GradB=\nabla_B f(B, W) = \sum_{t=1}^T {\phimt{t}{1}}^\top (\phimt{t}{1} B w_t - \ymt{t}{1}) w_t^\top\\
&\scalemath{0.95}{=\sum_{t=1}^T\sum_{n=0}^{\pG_1} (((1 - \rho) x_{b_\nt} + \rho \zeta_\nt)^\top B w_t - y_\nt) ((1 - \rho) x_{b_\nt} + \rho \zeta_\nt) w_t^\top}\\
&- (1 - \rho) \eta_\nt x_{b_\nt} w_t^\top - \rho \eta_\nt \zeta_\nt w_t^\top \\
&=(1 - \rho)^2 \GradB_1 + \rho (1 - \rho) (\GradB_2 + \GradB_3) + \rho^2 \GradB_4 \\
&-\sum_{t=1}^T \sum_{n=0}^{\pG_1} (1 - \rho) \eta_\nt x_{b_\nt} w_t^\top - \sum_{t=1}^T \sum_{n=0}^{\pG_1} \rho \eta_\nt \zeta_\nt w_t^\top.
\end{align*}
\section{Supporting Results and Proofs}
\begin{prop} \label{p1}
Let $M = B^\top {\phimt{t}{1}}^\top {\phimt{t}{1}} B$. For all $t \in [T]$, the following statements hold: 
\begin{itemize}
    \item With probability at least $1 - 4 T \exp(r - c \pG_1)$, 
    $$
    \|M^{-1}\| \leqslant \frac{2}{(1 - 2 \rho)^2 \pG_1}
    $$
    \item With probability at least $1 - 8 T \exp(r - c \epsilon_2^2 \pG_1)$, 
    $$
    \|M^{-1} B^\top {\phimt{t}{1}}^\top {\phimt{t}{1}} (I - B B^\top) \thetats\| \leqslant \frac{\|(I - B B^\top) \thetats\|}{(1 - 2 \rho)^2}.
    $$
\end{itemize}
\end{prop}
\begin{proof}
Consider a fixed $z \in \pS_r$. We have $z^\top B^\top {\phimt{t}{1}}^\top {\phimt{t}{1}} B z$
\begin{align*}
&=\sum_{n=0}^{\pG_1} z^\top B^\top [(1 - \rho) x_{b_\nt} + \rho \zeta_\nt] [(1 - \rho) x_{b_\nt} + \rho \zeta_\nt]^\top B z
\end{align*}
Given that $\bE[(1 - \rho) z^\top B^\top x_{b_\nt}]=0$,
\begin{align*}
\bE[(1 - \rho)^2 z^\top B^\top x_{b_\nt} x_{b_\nt}^\top B z]&=(1 - \rho)^2 z^\top B^\top \bE[x_{b_\nt} x_{b_\nt}^\top] B z\\
&=(1 - \rho)^2 z^\top B^\top B z=(1 - \rho)^2,
\end{align*}
\vspace{-6mm}
\begin{align*}
\Var((1 - \rho) z^\top B^\top x_{b_\nt})&=\bE[(1 - \rho) z^\top B^\top x_{b_\nt}]^2\\
&=(1 - \rho)^2 z^\top B^\top \bE[x_{b_\nt} x_{b_\nt}^\top] B z=(1 - \rho)^2.
\end{align*}
The summands are presented as independent sub-exponential random variables with a bounded norm $K_n \leqslant C (1 - \rho)^2$. We apply the sub-exponential Bernstein inequality, as specified in Proposition~\ref{proposition1}, by setting $g = \epsilon_2 (1 - \rho)^2 \pG_1$. We have
\begin{align*}
\frac{g^2}{\sum_{n=0}^{\pG_1} K_n^2} &\geqslant \frac{\epsilon_2^2 (1 - \rho)^4 \pG_1^2}{\pG_1 C^2 (1 - \rho)^4} \geqslant c \epsilon_2^2 \pG_1 \\
\frac{g}{\max_n K_n} &\geqslant \frac{\epsilon_2 (1 - \rho)^2 \pG_1}{C (1 - \rho)^2} \geqslant c \epsilon_2 \pG_1
\end{align*}
For a fixed $z \in \pS_r$, with probability at least $1 - \exp(- c \epsilon_2^2 \pG_1)$, 
\begin{align*}
\sum_{n=0}^{\pG_1} (1 - \rho)^2 z^\top B^\top x_{b_\nt} x_{b_\nt}^\top B z \geqslant (1 - \epsilon_2) (1 - \rho)^2 \pG_1
\end{align*}
By extending our analysis over the entire space $\pS_r$ through an epsilon-net, the result adds a multiplicative factor of $\exp(r)$. Therefore, with probability at least $1 - \exp(r - c \epsilon_2^2 \pG_1)$,
$$
\min_{z \in \pS_r} \sum_{n=0}^{\pG_1} (1 - \rho)^2 z^\top B^\top x_{b_\nt} x_{b_\nt}^\top B z \geqslant (1 - \rho)^2 (1 - \epsilon_2) \pG_1
$$
Building on this idea, with the same probability, we derive
\begin{align*}
&\min_{z \in \pS_r} \sum_{n=0}^{\pG_1} \rho (1 - \rho) z^\top B^\top x_{b_\nt} \zeta_\nt^\top B z \geqslant - \rho (1 - \rho) \epsilon_2 \pG_1\\
&\min_{z \in \pS_r} \sum_{n=0}^{\pG_1} \rho (1 - \rho) z^\top B^\top \zeta_\nt x_{b_\nt}^\top B z \geqslant - \rho (1 - \rho) \epsilon_2 \pG_1\\
&\min_{z \in \pS_r} \sum_{n=0}^{\pG_1} \rho^2 z^\top B^\top \zeta_\nt \zeta_\nt^\top B z \geqslant \rho^2 (1 - \epsilon_2) \pG_1
\end{align*}
By using the union bound and setting $\epsilon_2 = 0.5$, with probability at least $1 - 4 T \exp(r - c \pG_1)$, for all $t \in [T]$, we conclude that
\begin{align*}
&\|M^{-1}\|=\|(B^\top {\phimt{t}{1}}^\top {\phimt{t}{1}} B)^{-1}\| \\
&=\frac{1}{\sigma_{\min}(B^\top {\phimt{t}{1}}^\top {\phimt{t}{1}} B)}=\frac{1}{\min_{z \in \pS_r} \sum_{n=0}^{\pG_1} z^T B^\top {\phimt{t}{1}}^\top {\phimt{t}{1}} B z}\\
&\scalemath{0.9}{\leqslant\frac{1}{(1 - \rho)^2 (1 - \epsilon_2) \pG_1 - \rho (1 - \rho) \epsilon_2 \pG_1 - \rho (1 - \rho) \epsilon_2 \pG_1 + \rho^2 (1 - \epsilon_2) \pG_1}}\\
&=\frac{2}{(1 - 2 \rho)^2 \pG_1}.
\end{align*}
This proves the first statement. Now consider a fixed $z \in \pS_r$,
\begin{align*}
&z^\top B^\top {\phimt{t}{1}}^\top {\phimt{t}{1}} (I - B B^\top) \thetats\\
&=\scalemath{0.95}{\sum_{n=0}^{\pG_1} z^\top B^\top [(1 - \rho) x_{b_\nt}+\rho \zeta_\nt][(1 - \rho) x_{b_\nt} + \rho \zeta_\nt]^\top (I - B B^\top) \thetats}
\end{align*}
Given that $\bE[(1 - \rho) x_{b_\nt}^\top (I - B B^\top) \thetats]=0$,
\begin{align*}
&\bE[(1 - \rho) z^\top B^\top x_{b_\nt}]=0,\quad\Var((1 - \rho) z^\top B^\top x_{b_\nt}) = (1 - \rho)^2,\\
&\bE[(1 - \rho)^2 z^\top B^\top x_{b_\nt} x_{b_\nt}^\top (I - B B^\top) \thetats]=0\\
&\Var((1 - \rho) x_{b_\nt}^\top (I - B B^\top) \thetats)\\
&=\bE[(1 - \rho)^2 {\thetats}^\top (I - B B^\top)^\top x_{b_\nt} x_{b_\nt}^\top (I - B B^\top) \thetats]\\
&\scalemath{0.95}{=(1 - \rho)^2 {\thetats}^\top (I - B B^\top)^\top (I - B B^\top) \thetats=(1 - \rho)^2 \|(I - B B^\top) \thetats\|^2.}
\end{align*}
As before, the summands are independent sub-exponential random variables with norm $K_n \leqslant C (1 - \rho)^2 \|(I - B B^\top) \thetats\|$. We use the sub-exponential Bernstein inequality stated in Proposition~\ref{proposition1}, set $g = \epsilon_2 (1 - \rho)^2 \|(I - B B^\top) \thetats\| \pG_1$. We have
\begin{align*}
\frac{g^2}{\sum_{n=0}^{\pG_1} K_n^2} &\geqslant \frac{\epsilon_2^2 (1 - \rho)^4 \|(I - B B^\top) \thetats\|^2 \pG_1^2}{\pG_1 C^2 (1 - \rho)^4 \|(I - B B^\top) \thetats\|^2} \geqslant c \epsilon_2^2 \pG_1, \\
\frac{g}{\max_n K_n} &\geqslant \frac{\epsilon_2 (1 - \rho)^2 \|(I - B B^\top) \thetats\| \pG_1}{C (1 - \rho)^2 \|(I - B B^\top) \thetats\|} \geqslant c \epsilon_2 \pG_1.
\end{align*}
For a fixed $z \in \pS_r$, with probability at least $1 - \exp(- c \epsilon_2^2 \pG_1)$, 
\begin{align*}
\scalemath{0.9}{\sum_{n=0}^{\pG_1} (1 - \rho)^2 z^\top B^\top x_{b_\nt} x_{b_\nt}^\top (I - B B^\top) \thetats\leqslant\epsilon_2 (1 - \rho)^2 \|(I - B B^\top) \thetats\| \pG_1.}
\end{align*}
Extending this to include all $\pS_r$ through an epsilon-net introduces a factor of $\exp(r)$. Consequently, with probability at least $1 - \exp(r - c \epsilon_2^2 \pG_1)$, we obtain
\begin{align*}
&\max_{z \in \pS_r} \sum_{n=0}^{\pG_1} (1 - \rho)^2 z^\top B^\top x_{b_\nt} x_{b_\nt}^\top (I - B B^\top) \thetats\\
&\leqslant\epsilon_2 (1 - \rho)^2 \|(I - B B^\top) \thetats\| \pG_1.
\end{align*}
Using this approach, with the same probability, we have
\begin{align*}
&\max_{z \in \pS_r} \sum_{n=0}^{\pG_1} \rho (1 - \rho) z^\top B^\top x_{b_\nt} \zeta_\nt^\top (I - B B^\top) \thetats\\
&\leqslant\epsilon_2 \rho (1 - \rho) \|(I - B B^\top) \thetats\| \pG_1,\\
&\max_{z \in \pS_r} \sum_{n=0}^{\pG_1} \rho (1 - \rho) z^\top B^\top \zeta_\nt x_{b_\nt}^\top (I - B B^\top) \thetats\\
&\leqslant \epsilon_2 \rho (1 - \rho) \|(I - B B^\top) \thetats\| \pG_1,
\end{align*}
\begin{align*}
&\max_{z \in \pS_r} \sum_{n=0}^{\pG_1} \rho^2 z^\top B^\top \zeta_\nt \zeta_\nt^\top (I - B B^\top) \thetats \leqslant \epsilon_2 \rho^2 \|(I - B B^\top) \thetats\| \pG_1.
\end{align*}
Using the union bound, with probability at least $1 - 4 \exp(r - c \pG_1)$, for $\epsilon_2=0.5$, we conclude that
\begin{align*}
&\|B^\top {\phimt{t}{1}}^\top {\phimt{t}{1}} (I - B B^\top) \thetats\|=\max_{z \in \pS_r} z^\top B^\top {\phimt{t}{1}}^\top {\phimt{t}{1}} (I - B B^\top) \thetats\\
&\leqslant\epsilon_2 (1 - \rho)^2 \|(I - B B^\top) \thetats\| \pG_1 + \epsilon_2 \rho (1 - \rho) \|(I - B B^\top) \thetats\| \pG_1\\
&+\epsilon_2 \rho (1 - \rho) \|(I - B B^\top) \thetats\| \pG_1 + \epsilon_2 \rho^2 \|(I - B B^\top) \thetats\| \pG_1\\
&=\frac{1}{2} \|(I - B B^\top) \thetats\| \pG_1.
\end{align*}
By combining these results and using the union bound, for all $t \in [T]$, with probability at least $1 - 8 T \exp(r - c \pG_1)$, we have
\begin{align*}
&\|M^{-1} B^\top {\phimt{t}{1}}^\top {\phimt{t}{1}} (I - B B^\top) \thetats\|\\
&\leqslant \|M^{-1}\| \|B^\top {\phimt{t}{1}}^\top {\phimt{t}{1}} (I - B B^\top) \thetats\|\\
&\leqslant\frac{2}{(1 - 2 \rho)^2 \pG_1} \frac{1}{2} \|(I - B B^\top) \thetats\| \pG_1=\frac{\|(I - B B^\top) \thetats\|}{(1 - 2 \rho)^2}.
\end{align*}
This proves the second statement and completes the proof.
\end{proof}

\begin{lemma} \label{l2}
If $\delt \leqslant \frac{0.02}{\mu \sqrt{r} \kappa^2}$, then with probability at least $1 - 8 T \exp(r - c \pG_1) - 2 T \exp(r - \frac{c \epsilon_3^2 \pG_1}{\sigma_\eta^2})$, the following hold:
\begin{enumerate}
    \item $\| w_t - g_t \| \leqslant \frac{2}{(1 - 2 \rho)^2} \mu \sqrt{\frac{r}{T}} \delt \sigma_{\max}^\star$
    \item $\| w_t \| \leqslant \left(1 + \frac{0.04}{(1 - 2 \rho)^2}\right) \mu \sqrt{\frac{r}{T}} \sigma_{\max}^\star$
    \item $\| W - G \|_F \leqslant \frac{2}{(1 - 2 \rho)^2} \mu \sqrt{r} \delt \sigma_{\max}^\star$
    \item $\| \theta_t - \thetats \| \leqslant \left(1 + \frac{2}{(1 - 2 \rho)^2}\right) \mu \delt \sqrt{\frac{r}{T}} \sigma_{\max}^\star$
    \item $\| \Theta_\ell - \Thetas\|_F \leqslant \left(1 + \frac{2}{(1 - 2 \rho)^2}\right) \mu \delt \sqrt{r} \sigma_{\max}^\star$
    \item $\sigma_{\min}(W) \geqslant \left(\sqrt{1 - 4 \times 10^{-4}} - \frac{0.04}{(1 - 2 \rho)^2}\right) \sigma_{\min}^\star$
    \item $\sigma_{\max}(W) \leqslant \left(1 + \frac{0.04}{(1 - 2 \rho)^2}\right) {\sigma_{\max}^\star}$
\end{enumerate}
\end{lemma}
\begin{proof}
From Algorithm~\ref{alg1}, we have
\begin{align*}
w_t&=(\phimt{t}{1} B)^\dagger \ymt{t}{1}=((\phimt{t}{1} B)^\top (\phimt{t}{1} B))^{-1} (\phimt{t}{1} B)^\top \ymt{t}{1}\\
&=(B^\top {\phimt{t}{1}}^\top \phimt{t}{1} B)^{-1} (B^\top {\phimt{t}{1}}^\top) \phimt{t}{1} B B^\top \thetats\\
&+(B^\top {\phimt{t}{1}}^\top \phimt{t}{1} B)^{-1} (B^\top {\phimt{t}{1}}^\top) \phimt{t}{1} (I - B B^\top) \thetats\\
&+(B^\top {\phimt{t}{1}}^\top \phimt{t}{1} B)^{-1} (B^\top {\phimt{t}{1}}^\top) \etamt{t}{1}\\
&=g_t + M^{-1} B^\top {\phimt{t}{1}}^\top \phimt{t}{1} (I - B B^\top) \thetats + M^{-1} B^\top {\phimt{t}{1}}^\top \etamt{t}{1}. 
\end{align*}
We know $w_t - g_t = M^{-1} B^\top {\phimt{t}{1}}^\top \phimt{t}{1} (I - B B^\top) \thetats + M^{-1} B^\top {\phimt{t}{1}}^\top \etamt{t}{1}$. To bound $w_t - g_t$, we also need to find the upper bound of $\|B^\top {\phimt{t}{1}}^\top \etamt{t}{1}\|$. 
We know 
\begin{align*}
&\|B^\top {\phimt{t}{1}}^\top \etamt{t}{1}\|=\max_{z \in \pS_r} z^\top B^\top {\phimt{t}{1}}^\top \etamt{t}{1}\\
&=\max_{z \in \pS_r} \sum_{n=0}^{\pG_1} z^\top B^\top [(1 - \rho) x_{b_\nt} + \rho \zeta_\nt] \eta_\nt\\
&=\max_{z \in \pS_r} \sum_{n=0}^{\pG_1} (1 - \rho) z^\top B^\top x_{b_\nt} \eta_\nt + \max_{z \in \pS_r} \sum_{n=0}^{\pG_1} \rho z^\top B^\top \zeta_\nt \eta_\nt. 
\end{align*}
Considering a fixed $z \in \pS_r$. Given that 
\begin{align*}
&\bE[(1 - \rho) z^\top B^\top x_{b_\nt} \eta_\nt] = 0, \quad \bE[\eta_\nt] = 0,
\end{align*}
\begin{align*}
&\bE[(1 - \rho) z^\top B^\top x_{b_\nt}] = 0, \quad \Var(\eta_\nt) = \sigma_\eta^2,\\
&\Var((1 - \rho) z^\top B^\top x_{b_\nt})=\bE[(1 - \rho) z^\top B^\top x_{b_\nt}]^2\\
&=\bE[(1 - \rho)^2 z^\top B^\top x_{b_\nt} x_{b_\nt}^\top B z]=(1 - \rho)^2 z^\top B^\top \bE[x_{b_\nt} x_{b_\nt}^\top] B z\\
&=(1 - \rho)^2 z^\top B^\top B z=(1 - \rho)^2, 
\end{align*}
the summands are presented as independent sub-exponential random variables with a bounded norm $K_n \leqslant C (1 - \rho) \sigma_\eta$. We apply the sub-exponential Bernstein inequality from Proposition~\ref{proposition1}with $g = \epsilon_3 (1 - \rho) \pG_1$. We have
\begin{align*}
\frac{g^2}{\sum_{n=0}^{\pG_1} K_n^2} &\geqslant \frac{\epsilon_3^2 (1 - \rho)^2 \pG_1^2}{\pG_1 C^2 (1 - \rho)^2 \sigma_\eta^2} \geqslant \frac{c \epsilon_3^2 \pG_1}{\sigma_\eta^2},\\
\frac{g}{\max_n K_n} &\geqslant \frac{\epsilon_3 (1 - \rho) \pG_1}{C (1 - \rho) \sigma_\eta} \geqslant \frac{c \epsilon_3 \pG_1}{\sigma_\eta}.
\end{align*}
Considering $\epsilon_3 \leqslant \sigma_\eta$. Consequently, for a fixed $z \in \pS_r$, with probability at least $1 - \exp(- \frac{c \epsilon_3^2 \pG_1}{\sigma_\eta^2})$, we derive
\vspace{-2mm}
\begin{align*}
\sum_{n=0}^{\pG_1} (1 - \rho) z^\top B^\top x_{b_\nt} \eta_\nt \leqslant \epsilon_3 (1 - \rho) \pG_1.
\end{align*}
By extending this over $\pS_r$ using an epsilon-net, the result adds a multiplicative factor of $\exp(r)$. Therefore, with probability at least $1 - \exp(r - \frac{c \epsilon_3^2 \pG_1}{\sigma_\eta^2})$, we have
\vspace{-2mm}
\begin{align*}
\max_{z \in \pS_r} \sum_{n=0}^{\pG_1} (1 - \rho) z^\top B^\top x_{b_\nt} \eta_\nt \leqslant \epsilon_3 (1 - \rho) \pG_1.
\end{align*}
Following the same idea, we determine that with probability at least $1 - \exp(r - \frac{c \epsilon_3^2 \pG_1}{\sigma_\eta^2})$, we have
\vspace{-2mm}
\begin{align*}
\max_{z \in \pS_r} \sum_{n=0}^{\pG_1} \rho z^\top B^\top \zeta_\nt \eta_\nt \leqslant \epsilon_3 \rho \pG_1.
\end{align*}
Using union bound, w.p. at least $1-2T\exp(r-\frac{c\epsilon_3^2\pG_1}{\sigma_\eta^2})$
\begin{align*}
\|B^\top {\phimt{t}{1}}^\top \etamt{t}{1}\| \leqslant \epsilon_3 (1 - \rho) \pG_1 + \epsilon_3 \rho \pG_1 = \epsilon_3 \pG_1.
\end{align*}
Therefore, by using the union bound and setting $\epsilon_3 = \frac{c}{2} \mu \sqrt{\frac{r}{T}} \delt \sigma_{\max}^\star$, with probability at least $1 - 8 T \exp(r - c \pG_1) - 2 T \exp(r - \frac{c \epsilon_3^2 \pG_1}{\sigma_\eta^2})$, we conclude that 
\begin{align}
&\|w_t - g_t\|\leqslant\frac{\|(I - B B^\top) \thetats\|}{(1 - 2 \rho)^2} + \frac{2 \epsilon_3}{(1 - 2 \rho)^2} \nonumber\\
&\leqslant\frac{1}{(1 - 2 \rho)^2} \mu \sqrt{\frac{r}{T}} \delt \sigma_{\max}^\star + \frac{1}{(1 - 2 \rho)^2} \mu \sqrt{\frac{r}{T}} \delt \sigma_{\max}^\star \label{l2_1}\\
&= \frac{2}{(1 - 2 \rho)^2} \mu \sqrt{\frac{r}{T}} \delt \sigma_{\max}^\star,\label{l2_2}
\end{align}
where Eq.~\eqref{l2_1} is derived from Assumption~\ref{assume:incoherence}. This proves 1). 
\begin{align}
\|W - G\|_F&\leqslant\frac{2}{(1 - 2 \rho)^2} \mu \sqrt{r} \delt \sigma_{\max}^\star \label{l2_3}\\
\|w_t\|&=\|w_t - g_t + g_t\|\leqslant \|w_t - g_t\| + \|w_t^\star\| \label{l2_4} \\
&\leqslant\frac{2}{(1 - 2 \rho)^2} \mu \sqrt{\frac{r}{T}} \delt \sigma_{\max}^\star + \mu \sqrt{\frac{r}{T}} \sigma_{\max}^\star \nonumber 
\end{align}
\begin{align}
&=(1 + \frac{0.04}{(1 - 2 \rho)^2})\mu\sqrt{\frac{r}{T}}\sigma_{\max}^\star. \nonumber
\end{align}
\begin{align}
\|\thetats - \thetahatt\|&=\|B g_t + (I - B B^\top) \thetats - B w_t\| \nonumber \\
&\leqslant(1+\frac{2}{(1-2\rho)^2})\mu\sqrt{\frac{r}{T}}\delt\sigma_{\max}^\star,\label{l2_5}
\end{align}
where Eq.~\eqref{l2_4} is derived from $\|g_t\| \leqslant \|w_t^\star\|$. This completes the proofs of 2), 3), and 4). Eq.~\eqref{l2_5} implies
\begin{align*}
\| \Theta_\ell - \Thetas\|_F \leqslant \left(1 + \frac{2}{(1 - 2 \rho)^2}\right) \mu \sqrt{r} \delt \sigma_{\max}^\star.
\end{align*}
This proves 5). Furthermore, 
\begin{align*}
&\sigma_{\min}(W)=\sigma_{\min}(G - (G - W))\\
&\geqslant \sigma_{\min}(G) - \| W - G \|\geqslant \sigma_{\min}(G) - \| W - G \|_F.
\end{align*}
Given that
\begin{align*}
\sigma_{\min}(G)&=\sigma_{\min}(G^\top)=\sigma_{\min}({W^\star}^\top {B^\star}^\top B)\geqslant \sigma_{\min}^\star \sigma_{\min}({B^\star}^\top B),\\
\sigma_{\min}({B^\star}^\top B)&=\sqrt{\lambda_{\min}(B^\top B^{\star} B^{\star^\top} B)})=\sqrt{\lambda_{\min}(B^\top (I - P) B)}\\
&=\sqrt{\lambda_{\min}(I - B^\top P B)}=\sqrt{\lambda_{\min}(I - B^\top P^2 B)}\\
&=\sqrt{1 - \lambda_{\max}(B^\top P^2 B)}=\sqrt{1 - \| P B \|^2}\geqslant \sqrt{1 - \delt^2}. 
\end{align*}
By combining the above with Eq.~\eqref{l2_4} and substituting $\delta_\ell$
\begin{align*}
&\sigma_{\min}(W) \geqslant \sqrt{1 - \delt^2} \sigma_{\min}^\star - \frac{2}{(1 - 2 \rho)^2} \mu \sqrt{r} \delt \sigma_{\max}^\star\\
&\geqslant \big(\sqrt{1 - 4 \times 10^{-4}} - \frac{0.04}{(1 - 2 \rho)^2}\big) \sigma_{\min}^\star\\
&\sigma_{\max}(W)=\sigma_{\max}(G - (G - W))\leqslant \sigma_{\max}(G) + \sigma_{\max}(G - W)\\
&=\sigma_{\max}(B^\top B^\star W^\star) + \sigma_{\max}(G - W)\\
&\leqslant\sigma_{\max}(B^\top B^\star) \sigma_{\max}(W^\star) + \| G - W \|_F\\
&\leqslant\sigma_{\max}^\star + \frac{2}{(1 - 2 \rho)^2} \mu \sqrt{r} \delt \sigma_{\max}^\star=\big(1 + \frac{0.04}{(1 - 2 \rho)^2}\big) \sigma_{\max}^\star.
\end{align*}
This completes the proofs of 6) and 7). 
\end{proof}
\begin{prop} \label{p3}
Assume $\SE(B, B^\star) \leqslant \delt$. Then, the following statements hold: 
\begin{itemize}
    \item We have $\bE[\GradB_1]=\bE[\GradB_4]=\pG_1 (\Theta - \Thetas) W^\top$ and $\bE[\GradB_2]=\bE[\GradB_3]=0$. 
    \item With probability at least $1 - 8 T \exp(r - c \pG_1) - 2 T \exp(r - \frac{c \epsilon_3^2 \pG_1}{\sigma_\eta^2})$, we have $\|\bE[\GradB_2]\|=\|\bE[\GradB_3]\|=0$ and
    \begin{align*}
    &\|\bE[\GradB_1]\|=\|\bE[\GradB_4]\|\\
    &\leqslant\Big(1 + \frac{0.04}{(1 - 2 \rho)^2}\Big) \Big(1 + \frac{2}{(1 - 2 \rho)^2}\Big) \pG_1 \mu \sqrt{r} \delt \sigma_{\max}^{\star^2}.
    \end{align*}
    \item If $\delt \leqslant \frac{0.02}{\mu \sqrt{r} \kappa^2}$, then with probability at least $1 - \exp(C (d + r) - c \frac{\epsilon_4^2 \pG_1 T}{\mu^2 r \kappa^4}) - 8 T \exp(r - c \pG_1) - 2 T \exp(r - \frac{c \epsilon_3^2 \pG_1}{\sigma_\eta^2})$, for $i \in \{1, 2, 3, 4\}$, we have
    \begin{align*}
    &\|\GradB_i - \bE[\GradB_i]\|\\
    &\leqslant \Big(1 + \frac{0.04}{(1 - 2 \rho)^2}\Big) \Big(1 + \frac{2}{(1 - 2 \rho)^2}\Big) \epsilon_4 \pG_1 \sigma_{\min}^{\star^2} \delt. 
    \end{align*}
\end{itemize}
\end{prop}
\begin{proof}
Starting with the definition, we have
\begin{align*}
&\bE[\GradB_1]=\bE[\sum_{t=1}^T \sum_{n=0}^{\pG_1} x_{b_\nt} x_{b_\nt}^\top (\theta_t - \thetats) w_t^\top]\\
&=\sum_{t=1}^T \sum_{n=0}^{\pG_1} \bE[x_{b_\nt} x_{b_\nt}^\top] (\theta_t - \thetats) w_t^\top=\sum_{t=1}^T \sum_{n=0}^{\pG_1} (\theta_t - \thetats) w_t^\top\\
&=\sum_{t=1}^T \pG_1 (\theta_t - \thetats) w_t^\top=\pG_1 (\Theta - \Thetas) W^\top.
\end{align*}
Using the same idea, we derive $\bE[\GradB_2]=\bE[\GradB_3]=0$,
\begin{align*}
\bE[\GradB_4]&=\bE[\sum_{t=1}^T \sum_{n=0}^{\pG_1} \zeta_\nt \zeta_\nt^\top (\theta_t - \thetats) w_t^\top]
= \pG_1 (\Theta - \Thetas) W^\top.
\end{align*}
Using the result from Lemma~\ref{l2}, with probability at least $1 - 8 T \exp(r - c \pG_1) - 2 T \exp(r - \frac{c \epsilon_3^2 \pG_1}{\sigma_\eta^2})$, we conclude
\begin{align*}
&\|\bE[\GradB_1]\|=\|\bE[\GradB_4]\|\leqslant \pG_1 \|(\Theta - \Thetas)\| \|W\|\\
&\leqslant\left(1 + \frac{0.04}{(1 - 2 \rho)^2}\right) \left(1 + \frac{2}{(1 - 2 \rho)^2}\right) \pG_1 \mu \sqrt{r} \delt \sigma_{\max}^{\star^2}\\
&\|\bE[\GradB_2]\| = \|\bE[\GradB_3]\| \leqslant \pG_1 \|(\Theta - \Thetas)\| \|W\|=0.
\end{align*}
Considering the fixed unit norm vectors $z$ and $v$. We have 
\begin{align*}
&\|\GradB_1 - \bE[\GradB_1]\|\\
&=\max_{\|z\|=1, \|v\|=1} z^\top \Big(\sum_{t=1}^T \sum_{n=0}^{\pG_1} x_{b_\nt} x_{b_\nt}^\top(\theta_t - \thetats) w_t^\top-\bE[\GradB_1]\Big) v\\
&=\max_{\|z\|=1, \|v\|=1} \Big(\sum_{t=1}^T \sum_{n=0}^{\pG_1} z^\top x_{b_\nt} w_t^\top v x_{b_\nt}^\top(\theta_t - \thetats)-z^\top \bE[\GradB_1] v \Big).
\end{align*}
Given that $\bE[\GradB_1 - \bE[\GradB_1]] = 0$, $\bE[z^\top x_{b_\nt} w_t^\top v] = 0$, and $\bE[x_{b_\nt}^\top (\theta_t - \thetats)] = 0$,
\begin{align*}
&\Var(z^\top x_{b_\nt} w_t^\top v)=\bE[(w_t^\top v)^2 z^\top x_{b_\nt} x_{b_\nt}^\top z]\\
&=(w_t^\top v)^2 z^\top \bE[x_{b_\nt}x_{b_\nt}^\top] z=(w_t^\top v)^2 z^\top z=(w_t^\top v)^2,\\
&\Var(x_{b_\nt}^\top (\theta_t - \thetats))=\bE[(\theta_t - \thetats)^\top x_{b_\nt}x_{b_\nt}^\top (\theta_t - \thetats)]\\
&\scalemath{0.9}{=(\theta_t - \thetats)^\top \bE[x_{b_\nt}x_{b_\nt}^\top] (\theta_t - \thetats)=(\theta_t - \thetats)^\top (\theta_t - \thetats)=\|\thetats - \theta_t\|^2,}
\end{align*}
the summands are presented as independent sub-exponential random variables with a bounded norm $K_n \leqslant C (w_t^\top v) \|\thetats - \theta_t\|$. We use the sub-exponential Bernstein inequality from Proposition~\ref{proposition1} with $g = \left(1 + \frac{0.04}{(1 - 2 \rho)^2}\right) \left(1 + \frac{2}{(1 - 2 \rho)^2}\right) \epsilon_4 \pG_1 \sigma_{\min}^{\star^2} \delt$. To apply this, based on the result of Lemma~\ref{l2}, we have
\begin{align*}
&\frac{g^2}{\sum_{n=0}^{\pG_1} K_n^2}\geqslant \frac{\left(1 + \frac{0.04}{(1 - 2 \rho)^2}\right)^2 \left(1 + \frac{2}{(1 - 2 \rho)^2}\right)^2 \epsilon_4^2 \pG_1^2 \sigma_{\min}^{\star^4} \delt^2}{\sum_{t=1}^T \sum_{n=0}^{\pG_1} C^2 (w_t^\top v)^2 \|\thetats - \theta_t\|^2}\\
&\geqslant \frac{\left(1 + \frac{0.04}{(1 - 2 \rho)^2}\right)^2 \left(1 + \frac{2}{(1 - 2 \rho)^2}\right)^2 \epsilon_4^2 \pG_1^2 \sigma_{\min}^{\star^4} \delt^2}{\pG_1 C^2 \max_t \|\thetats - \theta_t\|^2 \|v^\top W\|^2}\\
&\geqslant \frac{\left(1 + \frac{0.04}{(1 - 2 \rho)^2}\right)^2 \left(1 + \frac{2}{(1 - 2 \rho)^2}\right)^2 \epsilon_4^2 \pG_1^2 \sigma_{\min}^{\star^4} \delt^2}{\pG_1 C^2 \left(1 + \frac{2}{(1 - 2 \rho)^2}\right)^2 \mu^2 \frac{r}{T} \delt^2 \sigma_{\max}^{\star^2} \left(1 + \frac{0.04}{(1 - 2 \rho)^2}\right)^2 \sigma_{\max}^{\star^2}} \\
&\geqslant \frac{c \epsilon_4^2 \pG_1 T}{\mu^2 r \kappa^4}.
\end{align*}
\begin{align*}
&\frac{g}{\max_n K_n}\geqslant \frac{\left(1 + \frac{0.04}{(1 - 2 \rho)^2}\right) \left(1 + \frac{2}{(1 - 2 \rho)^2}\right) \epsilon_4 \pG_1 \sigma_{\min}^{\star^2} \delt}{\max_{n, t} C (w_t^\top v) \|\thetats - \theta_t\|} \\
&\geqslant \frac{\left(1 + \frac{0.04}{(1 - 2 \rho)^2}\right) \left(1 + \frac{2}{(1 - 2 \rho)^2}\right) \epsilon_4 \pG_1 \sigma_{\min}^{\star^2} \delt}{C \max_t \|w_t\| \max_t \|\thetats - \theta_t\|} \\
&\geqslant \frac{\left(1 + \frac{0.04}{(1 - 2 \rho)^2}\right) \left(1 + \frac{2}{(1 - 2 \rho)^2}\right) \epsilon_4 \pG_1 \sigma_{\min}^{\star^2} \delt}{C \left(1 + \frac{0.04}{(1 - 2 \rho)^2}\right) \mu \sqrt{\frac{r}{T}} \sigma_{\max}^\star \left(1 + \frac{2}{(1 - 2 \rho)^2}\right) \mu \sqrt{\frac{r}{T}} \delt \sigma_{\max}^\star} \\
&\geqslant \frac{c \epsilon_4 \pG_1 T}{\mu^2 r \kappa^2}.
\end{align*}
Thus, for a fixed unit norm vectors $z, v$, with probability at least $1 - \exp(- \frac{c\epsilon_4^2 \pG_1 T}{\mu^2 r \kappa^4}) - 8 T \exp(r - c \pG_1) - 2 T \exp(r - \frac{c \epsilon_3^2 \pG_1}{\sigma_\eta^2})$,
\begin{align*}
\scalemath{0.9}{z^\top (\GradB_1 - \bE[\GradB_1]) v\leqslant(1 + \frac{0.04}{(1 - 2 \rho)^2})(1 + \frac{2}{(1 - 2 \rho)^2}) \epsilon_4 \pG_1 \sigma_{\min}^{\star^2} \delt.}
\end{align*}
Applying a standard epsilon-net argument to bound the maximum of the above expression across all unit norm vectors $z$ and $v$ introduces a multiplicative factor of $\exp(C (d + r))$. Consequently, with probability at least $1 - \exp(C (d + r) - \frac{c\epsilon_4^2 \pG_1 T}{\mu^2 r \kappa^4}) - 8 T \exp(r - c \pG_1) - 2 T \exp(r - \frac{c \epsilon_3^2 \pG_1}{\sigma_\eta^2})$, it follows that
\vspace{-2mm}
\begin{align*}
\scalemath{0.9}{\|\GradB_1 - \bE[\GradB_1]\|\leqslant(1 + \frac{0.04}{(1 - 2 \rho)^2})(1 + \frac{2}{(1 - 2 \rho)^2}) \epsilon_4 \pG_1 \sigma_{\min}^{\star^2} \delt}.
\end{align*}
Following the same idea, we have the same bounds for $\GradB_2$, $\GradB_3$, and $\GradB_4$. We determine that with probability at least $1 - \exp(C (d + r) - \frac{c\epsilon_4^2 \pG_1 T}{\mu^2 r \kappa^4}) - 8 T \exp(r - c \pG_1) - 2 T \exp(r - \frac{c \epsilon_3^2 \pG_1}{\sigma_\eta^2})$, for $i \in \{2, 3, 4\}$, it holds that
\begin{align*}
\scalemath{0.9}{\|\GradB_i - \bE[\GradB_i]\|\leqslant (1 + \frac{0.04}{(1 - 2 \rho)^2})(1 + \frac{2}{(1 - 2 \rho)^2}) \epsilon_4 \pG_1 \sigma_{\min}^{\star^2} \delt}.
\end{align*}
This completes the proof. 
\end{proof}

\begin{lemma} \label{l4}
Assume that $\SE(B, B^\star) \leqslant \delt$. Then, the following statements hold: 
\begin{itemize}
    \item $\bE[\GradB] = (2 \rho^2 - 2 \rho + 1) \pG_1 (\Theta - \Thetas) W^\top$
    \item With probability at least $1 - 8 T \exp(r - c \pG_1) - 2 T \exp(r - \frac{c \epsilon_3^2 \pG_1}{\sigma_\eta^2})$, we derive the upper bound of $\|\bE[\GradB]\|$ by
    \begin{align*}
    \scalemath{0.95}{(2 \rho^2 - 2 \rho + 1)(1 + \frac{0.04}{(1 - 2 \rho)^2})(1 + \frac{2}{(1 - 2 \rho)^2}) \pG_1 \mu \sqrt{r} \delt \sigma_{\max}^{\star^2}}
    \end{align*}
    \item If $\delt \leqslant \frac{0.02}{\mu \sqrt{r} \kappa^2}$, then with probability at least $1 - 2 \exp(C (d + r) - \frac{c \epsilon_5^2 \pG_1 \sigma_{\min}^{\star^2}}{\sigma_\eta^2 \kappa^2}) - 4 \exp(C (d + r) - c \frac{\epsilon_4^2 \pG_1 T}{\mu^2 r \kappa^4}) - 8 T \exp(r - c \pG_1) - 2 T \exp(r - \frac{c \epsilon_3^2 \pG_1}{\sigma_\eta^2})$, 
    \begin{align*}
    \|\GradB - \bE[\GradB]\|&\leqslant 2 (1 + \frac{0.04}{(1 - 2 \rho)^2})\\
    &(1 + \frac{2}{(1 - 2 \rho)^2}) \epsilon_4 \delt \pG_1 \sigma_{\min}^{\star^2}.
    \end{align*}
\end{itemize}
\end{lemma}
\begin{proof}
Given the definition of $\GradB$, we have
\begin{align*}
&\scalemath{0.9}{\bE[\GradB]=(1 - \rho)^2 \bE[\GradB_1] + \rho (1 - \rho) (\bE[\GradB_2] + \bE[\GradB_3])}\\
&-\sum_{t=1}^T \sum_{n=0}^{\pG_1} (1 - \rho) \bE[\eta_\nt x_{b_\nt} w_t^\top] - \sum_{t=1}^T \sum_{n=0}^{\pG_1} \rho \bE[\eta_\nt \zeta_\nt w_t^\top]
\end{align*}
\begin{align*}
&=(1 - \rho)^2 \pG_1 (\Theta - \Thetas) W^\top + \rho^2 \pG_1 (\Theta - \Thetas) W^\top\\
&=(2 \rho^2 - 2 \rho + 1) \pG_1 (\Theta - \Thetas) W^\top.
\end{align*}
Using the result form Lemma~\ref{l2}, with probability at least $1 - 8 T \exp(r - c \pG_1) - 2 T \exp(r - \frac{c \epsilon_3^2 \pG_1}{\sigma_\eta^2})$, we have
\begin{align*}
&\|\bE[\GradB]\|=\|(2 \rho^2 - 2 \rho + 1) \pG_1 (\Theta - \Thetas) W^\top\|\\
&\leqslant(2 \rho^2 - 2 \rho + 1) \pG_1 \|(\Theta - \Thetas)\| \|W\|\\
&\leqslant\scalemath{0.95}{(2 \rho^2 - 2 \rho + 1)(1 + \frac{0.04}{(1 - 2 \rho)^2})(1 + \frac{2}{(1 - 2 \rho)^2}) \pG_1 \mu \sqrt{r} \delt \sigma_{\max}^{\star^2}}. 
\end{align*}
Next, we bound $\| \GradB - \bE[\GradB] \|$ by bounding terms $\|\sum_{t=1}^T \sum_{n=0}^{\pG_1} (1 - \rho) \eta_\nt x_{b_\nt} w_t^\top\|$ and $\|\sum_{t=1}^T \sum_{n=0}^{\pG_1} \rho \eta_\nt \zeta_\nt w_t^\top\|$. Considering the fixed unit norm vectors $z$ and $v$. Given that $\bE[(1 - \rho) \eta_\nt v^\top x_{b_\nt} w_t^\top z] = 0$, $\bE[\eta_\nt] = 0, \quad \bE[(1 - \rho) v^\top x_{b_\nt} w_t^\top z] = 0$, and $\Var(\eta_\nt) = \sigma_\eta^2$,
\begin{align*}
&\Var((1 - \rho) v^\top x_{b_\nt} w_t^\top z)=\bE[(1 - \rho) v^\top x_{b_\nt} w_t^\top z]^2\\
&=\scalemath{0.95}{\bE[(1 - \rho)^2 (w_t^\top z)^2 v^\top x_{b_\nt}x_{b_\nt}^\top v]=(1 - \rho)^2 (w_t^\top z)^2 v^\top \bE[x_{b_\nt}x_{b_\nt}^\top] v}\\
&=(1 - \rho)^2 (w_t^\top z)^2 v^\top v=(1 - \rho)^2 (w_t^\top z)^2, 
\end{align*}
the summands are presented as independent sub-exponential random variables with a bounded norm $K_n \leqslant C (1 - \rho) |w_t^\top z| \sigma_\eta$. Applying the sub-exponential Bernstein inequality from Proposition~\ref{proposition1} with $g = \epsilon_5 (1 - \rho) \left(1 + \frac{0.04}{(1 - 2 \rho)^2}\right) \pG_1 \sigma_{\min}^{\star^2}$. To apply this, based on the result of Lemma~\ref{l2}, we have
\begin{align*}
&\frac{g^2}{\sum_{n=0}^{\pG_1} K_n^2}\geqslant\frac{\epsilon_5^2 (1 - \rho)^2 \big(1 + \frac{0.04}{(1 - 2 \rho)^2}\big)^2 \pG_1^2 \sigma_{\min}^{\star^4}}{\sum_{t=1}^T \sum_{n=0}^{\pG_1} C^2 (1 - \rho)^2 |w_t^\top z|^2 \sigma_\eta^2} \\
&\geqslant \frac{\epsilon_5^2 (1 - \rho)^2 \big(1 + \frac{0.04}{(1 - 2 \rho)^2}\big)^2 \pG_1^2 \sigma_{\min}^{\star^4}}{\pG_1 C^2 (1 - \rho)^2 \|W\|^2 \sigma_\eta^2}\\
&\geqslant \frac{\epsilon_5^2 (1 - \rho)^2 \big(1 + \frac{0.04}{(1 - 2 \rho)^2}\big)^2 \pG_1^2 \sigma_{\min}^{\star^4}}{\pG_1 C^2 (1 - \rho)^2 \sigma_\eta^2 \big(1 + \frac{0.04}{(1 - 2 \rho)^2}\big)^2 {\sigma_{\max}^\star}^2}\geqslant\frac{c \epsilon_5^2 \pG_1 \sigma_{\min}^{\star^2}}{\sigma_\eta^2 \kappa^2}\\
&\frac{g}{\max_n K_n}\geqslant\frac{\epsilon_5 (1 - \rho) \big(1 + \frac{0.04}{(1 - 2 \rho)^2}\big) \pG_1 \sigma_{\min}^{\star^2}}{\max_{n, t} C (1 - \rho) |w_t^\top z| \sigma_\eta}\\
&\geqslant \frac{\epsilon_5 (1 - \rho) \big(1 + \frac{0.04}{(1 - 2 \rho)^2}\big) \pG_1 \sigma_{\min}^{\star^2}}{C (1 - \rho) \|w_t\| \sigma_\eta} \\
&\geqslant \frac{\epsilon_5 (1 - \rho) \big(1 + \frac{0.04}{(1 - 2 \rho)^2}\big) \pG_1 \sigma_{\min}^{\star^2}}{C (1 - \rho) \big(1 + \frac{0.04}{(1 - 2 \rho)^2}\big) \mu \sqrt{\frac{r}{T}} \sigma_{\max}^\star \sigma_\eta}\geqslant\frac{c \epsilon_5 \pG_1 \sigma_{\min}^{\star}}{\mu \sqrt{\frac{r}{T}} \kappa \sigma_\eta}. 
\end{align*}
By considering $\epsilon_5 = c \epsilon_4 \delt$ and $\frac{c}{\kappa^3} \leqslant \frac{\sqrt{T} \sigma_\eta}{\sigma_{\min}^\star}$, for a fixed unit norm vectors $z, v$, with probability at least $1 - \exp(- \frac{c \epsilon_5^2 \pG_1 \sigma_{\min}^{\star^2}}{\sigma_\eta^2 \kappa^2}) - 8 T \exp(r - c \pG_1) - 2 T \exp(r - \frac{c \epsilon_3^2 \pG_1}{\sigma_\eta^2})$,
\vspace{-2mm}
\begin{align*}
\sum_{t=1}^T \sum_{n=0}^{\pG_1} (1 - \rho) \eta_\nt v^\top x_{b_\nt} w_t^\top z\leqslant \epsilon_5 (1 - \rho) \big(1 + \frac{0.04}{(1 - 2 \rho)^2}\big) \pG_1 \sigma_{\min}^{\star^2}.
\end{align*}
Applying a standard epsilon-net argument to bound the maximum of the above expression across all unit norm vectors $z$ and $v$ introduces a multiplicative factor of $\exp(C (d + r))$. Consequently, with probability at least $1 - \exp(C (d + r) - \frac{c \epsilon_5^2 \pG_1 \sigma_{\min}^{\star^2}}{\sigma_\eta^2 \kappa^2}) - 8 T \exp(r - c \pG_1) - 2 T \exp(r - \frac{c \epsilon_3^2 \pG_1}{\sigma_\eta^2})$, it follows that
\begin{align*}
\|\sum_{t=1}^T \sum_{n=0}^{\pG_1} (1 - \rho) \eta_\nt x_{b_\nt} w_t^\top\|\leqslant\epsilon_5 (1 - \rho) \big(1 + \frac{0.04}{(1 - 2 \rho)^2}\big) \pG_1 \sigma_{\min}^{\star^2}.
\end{align*}
Following the same idea, with the same probability, we have
\vspace{-2mm}
\begin{align*}
\|\sum_{t=1}^T \sum_{n=0}^{\pG_1} \rho \eta_\nt \zeta_\nt w_t^\top\| \leqslant \epsilon_5 \rho \big(1 + \frac{0.04}{(1 - 2 \rho)^2}\big) \pG_1 \sigma_{\min}^{\star^2}.
\end{align*}
By combining these results and Proposition~\ref{p3}, with probability at least $1 - 2 \exp(C (d + r) - \frac{c \epsilon_5^2 \pG_1 \sigma_{\min}^{\star^2}}{\sigma_\eta^2 \kappa^2}) - 4 \exp(C (d + r) - c \frac{\epsilon_4^2 \pG_1 T}{\mu^2 r \kappa^4}) - 8 T \exp(r - c \pG_1) - 2 T \exp(r - \frac{c \epsilon_3^2 \pG_1}{\sigma_\eta^2})$, we show that 
\begin{align*}
&\|\GradB - \bE[\GradB]\|\\
&\leqslant(1 - \rho)^2 \|\GradB_1 - \bE[\GradB_1]\|+\rho^2 \|\GradB_4 - \bE[\GradB_4]\\
&+\rho (1 - \rho) (\|\GradB_2 - \bE[\GradB_2]\| + \|\GradB_3 - \bE[\GradB_3]\|)\\
&+\Big\|\sum_{t=1}^T \sum_{n=0}^{\pG_1} (1 - \rho) \eta_\nt x_{b_\nt} w_t^\top\Big\| + \Big\|\sum_{t=1}^T \sum_{n=0}^{\pG_1} \rho \eta_\nt \zeta_\nt w_t^\top\Big\|\\
&\leqslant2\Big(1+\frac{0.04}{(1-2\rho)^2}\Big)\Big(1+\frac{2}{(1-2\rho)^2}\Big)\epsilon_4 \delt\pG_1\sigma_{\min}^{\star^2}.
\end{align*}
\end{proof}

\begin{lemma}\label{l8}
Let $\rho \in [0, \frac{\alpha\rbnt}{\rbnt+\sqrt{2\frac{r}{T}\log\frac{\pG_1 T}{\delta}}\mu\sigma_{\max}^\star}]$. Then, with probability at least $1-\delta$, the conservative feature vector is safe. 
\end{lemma}
\begin{proof}
To determine the safety of the conservative feature vector, it is sufficient to demonstrate that
\begin{equation}
(1 - \rho) x_{b_\nt}^\top \thetats + \rho \zeta_\nt^\top \thetats \geqslant (1 - \alpha) \rbnt. \nonumber
\end{equation}
This is equivalent to satisfy
\begin{equation}
\rho \leqslant \frac{\alpha \rbnt}{\rbnt + |\zeta_\nt^\top \thetats|}. \label{eq:rho_ineq}
\end{equation}
Since $\zeta_\nt$ follows standard Gaussian distribution, $\zeta_\nt^\top \thetats\sim\N(0,\|\thetats\|^2)$. Utilizing Gaussian tail bounds and union bound over $T$ tasks and $\pG_1$ rounds, with probability at least $1-\delta$
$$
\zeta_\nt^\top \thetats\leqslant\sqrt{2\log\frac{\pG_1 T}{\delta}}\|\thetats\|.
$$
To prove the inequality Eq.~\eqref{eq:rho_ineq}, we combine this result with Assumption~\ref{assume:incoherence}. It is sufficient to show that
\begin{equation}
\rho\leqslant\frac{\alpha\rbnt}{\rbnt+\sqrt{2\frac{r}{T}\log\frac{\pG_1 T}{\delta}}\mu\sigma_{\max}^\star}. \nonumber
\end{equation}
This provides the sufficient condition for the range of $\rho$.
\end{proof}

\subsection{Proof of Theorem~\ref{T5}} \label{proof_T5}
Based on the analysis from Theorem~5.2 in \cite{OurICML}, we have
\begin{align*}
\|P \wB^+\| &\leqslant \|P B\| \|I - \gamma W W^\top\| + \frac{\gamma}{\pG_1} \|\bE[\GradB] - \GradB\|, \\
\SE(B^+, B^\star) &\leqslant \frac{\|P \widehat{B}^+\|}{1 - \frac{\gamma}{\pG_1} \|\bE[\GradB]\| - \frac{\gamma}{\pG_1} \|\GradB - \bE[\GradB]\|}.
\end{align*}
Using the result from Lemma~\ref{l2}, with probability at least $1 - 8 T \exp(r - c \pG_1) - 2 T \exp(r - \frac{c \epsilon_3^2 \pG_1}{\sigma_\eta^2})$, we obtain
\vspace{-2mm}
\begin{align*}
\lambda_{\min}(I - \gamma W W^\top) = 1 - \gamma \|W\|^2 \geqslant 1 - \gamma \Big(1 + \frac{0.04}{(1 - 2 \rho)^2}\Big)^2 \sigma_{\max}^{\star^2},
\end{align*}
\vspace{-6mm}
\begin{align*}
\|I - \gamma W W^\top\|&=\lambda_{\max}(I - \gamma W W^\top)\\
&\leqslant 1 - \gamma \Big(\sqrt{1 - 4 \times 10^{-4}} - \frac{0.04}{(1 - 2 \rho)^2}\Big)^2 \sigma_{\min}^{\star^2}. 
\end{align*}
Since $\gamma \leqslant \frac{1}{\left(1 + \frac{0.04}{(1 - 2 \rho)^2}\right)^2 \sigma_{\max}^{\star^2}}$, $I - \gamma W W^\top$ is positive semi-definite. Combining the above with Lemma~\ref{l4} and setting $\epsilon_4 = 0.01$, with probability at least $1 - 2 \exp(C (d + r) - \frac{c \epsilon_5^2 \pG_1 \sigma_{\min}^{\star^2}}{\sigma_\eta^2 \kappa^2}) - 4 \exp(C (d + r) - c \frac{\epsilon_4^2 \pG_1 T}{\mu^2 r \kappa^4}) - 8 T \exp(r - c \pG_1) - 2 T \exp(r - \frac{c \epsilon_3^2 \pG_1}{\sigma_\eta^2})$, we have
\begin{align}
\|P \wB^+\|&\leqslant\|P B\| \|I - \gamma W W^\top\| + \frac{\gamma}{\pG_1} \|\bE[\GradB] - \GradB\| \nonumber \\
&\leqslant\Big[1 - \gamma \big(\sqrt{1 - 4 \times 10^{-4}} - \frac{0.04}{(1 - 2 \rho)^2}\big)^2 \sigma_{\min}^{\star^2}\Big] \delt \nonumber \\
&+2\gamma \big(1 + \frac{0.04}{(1 - 2 \rho)^2}\big) \big(1 + \frac{2}{(1 - 2 \rho)^2}\big) \epsilon_4 \delt \sigma_{\min}^{\star^2}\nonumber
\end{align}
\begin{align}
&\leqslant\Big(1 - \big[(\sqrt{1 - 4 \times 10^{-4}} - \frac{0.04}{(1 - 2 \rho)^2})^2\nonumber\\
&-2(1+\frac{0.04}{(1 - 2 \rho)^2})(1 + \frac{2}{(1 - 2 \rho)^2})\epsilon_4\big]\gamma\sigma_{\min}^{\star^2}\Big)\delt\nonumber\\
&\leqslant(1-0.5\gamma\sigma_{\min}^{\star^2})\delt,\nonumber\\
&\Big(1 - \frac{\gamma}{\pG_1} \|\bE[\GradB]\| - \frac{\gamma}{\pG_1} \|\GradB - \bE[\GradB]\|\Big)^{-1} \nonumber
\end{align}
\begin{align}
&\leqslant\Big(1 - (2 \rho^2 - 2 \rho + 1) \gamma (1 + \frac{0.04}{(1 - 2 \rho)^2}) (1 + \frac{2}{(1 - 2 \rho)^2})\mu\nonumber\\
&\delt \sqrt{r} \sigma_{\max}^{\star^2}-2\gamma(1 + \frac{0.04}{(1 - 2 \rho)^2})\sigma_{\min}^{\star^2}(1 + \frac{2}{(1 - 2 \rho)^2})\epsilon_4\delt\Big)^{-1} \nonumber\\
&\leqslant 1 + 2 (2 \rho^2 - 2 \rho + 1) \gamma(1 + \frac{0.04}{(1 - 2 \rho)^2})(1 + \frac{2}{(1 - 2 \rho)^2})\mu\nonumber\\
&\delt\sqrt{r}\sigma_{\max}^{\star^2}+4 \gamma (1 + \frac{0.04}{(1 - 2 \rho)^2}) \sigma_{\min}^{\star^2} (1 + \frac{2}{(1 - 2 \rho)^2}) \epsilon_4 \delt \label{T5_1}\\
&\leqslant 1+ 0.04(1 + \frac{0.04}{(1 - 2 \rho)^2}) \Big[(2 \rho^2 - 2 \rho + 1)(1 + \frac{2}{(1 - 2 \rho)^2})\nonumber\\
&+2\epsilon_4(1 + \frac{2}{(1 - 2 \rho)^2})\Big]\gamma \sigma_{\min}^{\star^2}\leqslant 1 + 0.27 \gamma \sigma_{\min}^{\star^2}, \nonumber
\end{align}
where Eq.~\eqref{T5_1} is derived from $(1 - x)^{-1} \leqslant (1 + 2 x)$ if $|x| \leqslant \frac{1}{2}$ for $\rho \in [0, 0.25] \cup [0.75, 1]$ and $\gamma\leqslant\frac{1}{(1 + \frac{0.04}{(1 - 2 \rho)^2})^2 \sigma_{\max}^{\star^2}}$. Combining these results above, we conclude that
\begin{align*}
\scalemath{0.9}{\SE(B^+, B^\star)\leqslant\big(1-0.5\gamma\sigma_{\min}^{\star^2}\big)\big(1+0.27\gamma\sigma_{\min}^{\star^2}\big)\delt\leqslant\big(1-0.23\gamma\sigma_{\min}^{\star^2}\big)\delt.}
\end{align*}
To ensure high probability guarantees for all stated results in the first epoch, it is necessary to set the bounds for $\pG_1$ and $\pG_1 T$. These bounds must guarantee that the following probability is sufficiently high: $1 - 2 \exp(C (d + r) - \frac{c \epsilon_5^2 \pG_1 \sigma_{\min}^{\star^2}}{\sigma_\eta^2 \kappa^2}) - 4 \exp(C (d + r) - c \frac{\epsilon_4^2 \pG_1 T}{\mu^2 r \kappa^4}) - 8 T \exp(r - c \pG_1) - 2 T \exp(r - \frac{c \epsilon_3^2 \pG_1}{\sigma_\eta^2})$. This required that each exponential term be substantially smaller than or equal to $d^{-10}$. By considering $\epsilon_3 = \frac{c}{2} \mu \sqrt{\frac{r}{T}} \delt \sigma_{\max}^\star$, $\epsilon_4 = 0.01$, $\epsilon_5 = c \epsilon_4 \delt$, and $\delta_L = \sqrt{2 (\rho^2 - \rho + 1)}$, we obtain
\begin{align*}
&C (d + r) - \frac{c \epsilon_5^2 \pG_1 \sigma_{\min}^{\star^2}}{\sigma_\eta^2 \kappa^2} < - 10 \log d \Rightarrow \pG_1 T > C (d + r) \kappa^2 \NSR \\
&C (d + r) - c \frac{\epsilon_4^2 \pG_1 T}{\mu^2 r \kappa^4} < - 10 \log d \Rightarrow \pG_1 T > C (d + r) \mu^2 r \kappa^4
\end{align*}
\begin{align*}
&\log T + r - c \pG_1 < - 10 \log d \Rightarrow \pG_1 > C (r + \log T + \log d) \\
&\scalemath{0.95}{\log T + r - c \frac{\epsilon_3^2 \pG_1}{\sigma_\eta^2} < - 10 \log d \Rightarrow \pG_1 > \frac{C (r + \log T + \log d)}{\mu^2 r \kappa^2} \NSR}.
\end{align*}
Consequently, combining these results, we conclude that
\begin{align*}
&\pG_1 T > C (d + r) \kappa^2 \max(\mu^2 r \kappa^2, \NSR) \\
&\pG_1 > C (r + \log T + \log d) \max(1, \frac{\NSR}{\mu^2 r \kappa^2}). 
\end{align*}
Moreover, when the initialization phase is considered, and by combining the above bounds with $\pG_1 T > C \mu^2 \kappa^2 \left(d r \frac{\kappa^2}{\delta_0^2} + \frac{d}{\delta_0^2} \NSR\right) = C^\prime \mu^4 r \kappa^6 d \left(r \kappa^2 + \NSR\right)$ from Proposition~\ref{p6}, with the condition $d + r \leqslant d r$, we prove that
\begin{align*}
&\pG_1 T > C \mu^4 d r \kappa^6 \max(r \kappa^2, \NSR) \\
&\pG_1 > C (r + \log T + \log d) \max(1, \frac{\NSR}{\mu^2 r \kappa^2}). 
\end{align*}
This completes the proof that ensures error decay in subsequent GD iterations in epoch-$1$ (safe exploration).
\qed

\subsection{Proof of Theorem~\ref{l9}} \label{proof_l9}
To guarantee the chosen action $x_\nt$ is safe, it is necessary that the following inequality is held: 
$$
x_\nt^\top \thetats \geqslant (1 - \alpha) \rbnt. 
$$
Adding and subtracting $x_\nt^\top \thetahatt^{\sm}$, we derive
$$
x_\nt^\top \thetahatt^{\sm} + x_\nt^\top (\thetats - \thetahatt^{\sm}) \geqslant (1 - \alpha) \rbnt. 
$$
Given that $x_\nt$ follows a standard Gaussian distribution, it follows that $x_\nt^\top(\thetats-\thetahatt^{\sm})\sim\N(0,\|\thetats-\thetahatt^{\sm}\|^2)$. Utilizing Gaussian tail bounds and the union bound over $T$ tasks and $\pG_{m+1}-\pG_m$ rounds, with probability at least $1-\delta$, we have
$$
x_\nt^\top(\thetats-\thetahatt^{\sm})\leqslant\sqrt{2\log\frac{(\pG_{m+1}-\pG_m) T}{\delta}}\|\thetats-\thetahatt^{\sm}\|.
$$
To guarantee the validity of this inequality, we combine this result with $\|\thetahatt^{\sm} - \thetats\| \leqslant \|\thetahatt^{\sone} - \thetats\|$ for $m \geqslant 1$. Consequently, to ensure constraints are met, we need to show that the following inequality holds.
$$
x_\nt^\top\thetahatt^{\sm}\geqslant\sqrt{2\log\frac{(\pG_{m+1}-\pG_m) T}{\delta}}\|\thetahatt^{\sone}-\thetats\|+(1-\alpha)\rbnt. 
$$
Using Lemma~\ref{l2} and Theorem~\ref{T5}, the inequality reduces to
\begin{align*}
x_\nt^\top \thetahatt^{\sm}\geqslant&(1+\frac{2}{(1-2\rho)^2})\mu(1-0.23\gamma\sigma_{\min}^{\star^2})^L\delta_0\sigma_{\max}^\star\\
&\sqrt{\frac{2r}{T}\log\frac{(\pG_{m+1}-\pG_m) T}{\delta}}+(1-\alpha)\rbnt.
\end{align*}
Given that $\delta_0 < 1$, $\rho\in[0,0.25]\cup[0.75,1]$, and by setting $\gamma = \frac{c}{(1 + \frac{0.04}{(1 - 2 \rho)^2})^2 \sigma_{\max}^{\star^2}}$, it suffices to ensure that
\begin{align*}
x_\nt^\top\thetahatt^{\sm}&\geqslant(1+\frac{2}{(1-2\rho)^2})\mu(1-\frac{c}{\kappa^2})^L\sigma_{\max}^\star \\
&\sqrt{\frac{2r}{T}\log\frac{(\pG_{m+1}-\pG_m) T}{\delta}}+(1-\alpha)\rbnt.
\end{align*}
Given the fact that $x_\nt^\top \thetahatt^{\sm} \geqslant x_\nt^{\star^\top} \thetahatt^{\sm}$, to ensure safety (constraint satisfaction) it suffices to verify that
\begin{align*}
x_\nt^{\star^\top}\thetats+x_\nt^{\star^\top}(\thetahatt^{\sm}-\thetats)\geqslant&(1-\alpha)\rbnt+\big(1+\frac{2}{(1-2\rho)^2}\big)\mu\sigma_{\max}^\star\\
&\big(1-\frac{c}{\kappa^2}\big)^L\sqrt{\frac{2r}{T}\log\frac{(\pG_{m+1}-\pG_m) T}{\delta}}.
\end{align*}
We can rewrite the safety condition as
\begin{align*}
x_\nt^{\star^\top}\thetats&\geqslant2(1+\frac{2}{(1-2\rho)^2})\mu(1 - \frac{c}{\kappa^2})^L\sigma_{\max}^\star\\
&\sqrt{\frac{2r}{T}\log\frac{(\pG_{m+1}-\pG_m)T}{\delta}}+(1-\alpha)\rbnt.
\end{align*}
Applying $\log x \geqslant 1 - \frac{1}{x}$, for $x\in (0,1]$, it is equivalent to satisfy
\begin{align*}
\big(1-\frac{c}{\kappa^2}\big)^L&\leqslant\frac{\kbnt+\alpha\rbnt}{2 \big(1+\frac{2}{(1-2\rho)^2}\big)\mu\sigma_{\max}^\star\sqrt{\frac{2r}{T}\log\frac{NT}{\delta}}}\\
L&\geqslant\frac{\log\Big(\frac{\kbnt+\alpha\rbnt}{2\big(1+\frac{2}{(1-2\rho)^2}\big)\mu\sigma_{\max}^\star\sqrt{\frac{2r}{T}\log\frac{NT}{\delta}}}\Big)}{\log\big(1-\frac{c}{\kappa^2}\big)}\\
L&\geqslant C\kappa^2\log\Big(\frac{\kbnt+\alpha\rbnt}{2\big(1+\frac{2}{(1-2\rho)^2}\big)\mu\sigma_{\max}^\star\sqrt{\frac{2r}{T}\log\frac{NT}{\delta}}}\Big). 
\end{align*}
Thus, we obtained the bound on the number of GD iterations to ensure safety. This completes the proof. 
\qed

\subsection{Proof of Theorem~\ref{T11}}\label{proof_T11}
This theorem is derived using Theorem~2.3 from \cite{singh2024noisy}. To obtain our result, we set $\epsilon_2 = c \epsilon_1 \delt$ in the original proof. This modification guarantees an exponential decay in $\SE(B^+, B^\star)$ at every iteration, given the specified conditions. 
\qed

\subsection{Proof of Theorem~\ref{T10}}\label{proof_T10}
We start by determining the upper bound for the cumulative regret $\pR_\NT^1$ during the first epoch, and $\pR_\NT^2$ in the following epochs. In the first epoch, since both $x_{b_\nt}$ and $\zeta_\nt$ are i.i.d standard Gaussian vectors, it follows that $\sum_{n=1}^{\pG_1} \left(x_\nt^\star - (1 - \rho) x_{b_\nt} - \rho \zeta_\nt\right)^\top \thetats \sim \N(0, 2 \pG_1 (\rho^2 - \rho + 1) \|\thetats\|^2)$. Using the Chernoff bound for Gaussian stated in Proposition~\ref{proposition2}, with probability at least $1 - \delta$, we have
\begin{align}
\scalemath{0.95}{\sum_{n=1}^{\pG_1} \left(x_\nt^\star - (1 - \rho) x_{b_\nt} - \rho \zeta_\nt\right)^\top \thetats\leqslant 2 \sqrt{\pG_1 (\rho^2 - \rho + 1) \log \frac{1}{\delta}} \|\thetats\|}\label{eq:reg-new-addition}. 
\end{align}
By using the union bound and combining the result from Theorem~\ref{T5}, with probability at least $1 - \delta - L d^{-10}$, we have

\begin{align}
\pR_\NT^1&=\sum_{t=1}^T \sum_{n=1}^{\pG_1} x_\nt^{\star^\top} \thetats - x_\nt^\top \thetats \nonumber \\
&= \sum_{t=1}^T \sum_{n=1}^{\pG_1} (x_\nt^\star - (1 - \rho) x_{b_\nt} - \rho \zeta_\nt)^\top \thetats \nonumber \\
&\leqslant \sum_{t=1}^T 2 \sqrt{\pG_1 (\rho^2 - \rho + 1) \log \frac{1}{\delta}} \|\thetats\| \label{T10_1_new} \\
&\leqslant 2 \mu \sigma_{\max}^\star \sqrt{\pG_1 r T (\rho^2 - \rho + 1) \log \frac{1}{\delta}}, \label{T10_1}
\end{align}
where Eq.~\eqref{T10_1} is derived from Assumption~\ref{assume:incoherence} and Eq.~\eqref{T10_1_new} is from Eq.~\eqref{eq:reg-new-addition}. By applying $L = C \kappa^2 \log \Big(\frac{\sqrt{2 (\rho^2 - \rho + 1)}}{1 + \frac{2}{(1 - 2 \rho)^2}}\Big)$, we demonstrate $\SE(B_L, B^\star) \leqslant \frac{\sqrt{2 (\rho^2 - \rho + 1)}}{1 + \frac{2}{(1 - 2 \rho)^2}}$. Subsequently, using the result from Lemma~\ref{l2}, we obtain 
\begin{align*}
\|\thetats - \thetahat_{1, t}\|&\leqslant\big(1 + \frac{2}{(1 - 2 \rho)^2}\big) \mu \delta_L \sqrt{\frac{r}{T}} \sigma_{\max}^\star\\
&\leqslant\sqrt{2 (\rho^2 - \rho + 1)} \mu \sqrt{\frac{r}{T}} \sigma_{\max}^\star. 
\end{align*}
Given that $\rho\in[0,0.25]\cup[0.75,1]$, we have $\sqrt{2(\rho^2-\rho+1)}>1$. Thus,
\vspace{-2mm}
\begin{align*}
&C\kappa^2\log\Big(\frac{\sqrt{2(\rho^2-\rho+1)}}{1+\frac{2}{(1-2 \rho)^2}}\Big)> C\kappa^2\log\Big(\frac{1}{1+\frac{2}{(1-2\rho)^2}}\Big)\\
&=C\kappa^2\log\Big(\frac{2\mu\sigma_{\max}^\star\sqrt{\frac{2r}{T}\log\frac{NT}{\delta}}}{2\big(1+\frac{2}{(1-2\rho)^2}\big)\mu\sigma_{\max}^\star\sqrt{\frac{2r}{T}\log\frac{NT}{\delta}}}\Big)\\
&\geqslant C\kappa^2\log\Big(\frac{2x_\nt^{\star^\top}\thetats}{2\big(1+\frac{2}{(1-2\rho)^2}\big)\mu\sigma_{\max}^\star\sqrt{\frac{2r}{T}\log\frac{NT}{\delta}}}\Big) \\
&\geqslant C\kappa^2\log\Big(\frac{\kbnt+\alpha\rbnt}{2\big(1+\frac{2}{(1-2\rho)^2}\big)\mu\sigma_{\max}^\star\sqrt{\frac{2r}{T}\log\frac{NT}{\delta}}}\Big),
\end{align*}
where the first inequality is derived with at least probability $1-\delta$, $x_\nt^{\star^\top}\thetats\leqslant\mu\sigma_{\max}^\star\sqrt{\frac{2r}{T}\log\frac{NT}{\delta}}$. Thus, the inequality above shows that the bound on $L$ satisfies the condition for $L$ in Theorem~\ref{l9}. For subsequent epochs, given $x_\nt$ as an i.i.d standard Gaussian, we derive that $\sum_{m=2}^M \sum_{n=\pG_{m-1}+1}^{\pG_m} x_\nt^{\star^\top} (\thetats - \thetahatmt) - x_\nt^\top (\thetats - \thetahatmt) \sim \N(0, 2 (N - \pG_1) \|\thetats - \thetahatmt\|^2)$. Using the Chernoff bound for Gaussian stated in Proposition~\ref{proposition2}, with probability at least $1 - \delta$, we have
\begin{align*}
&\sum_{m=2}^M \sum_{n=\pG_{m-1}+1}^{\pG_m} x_\nt^{\star^\top} (\thetats - \thetahatmt) - x_\nt^\top (\thetats - \thetahatmt)\\
&\leqslant 2 \sqrt{(N - \pG_1) \log \frac{1}{\delta}} \|\thetats - \thetahatmt\|. 
\end{align*}
By applying the union bound and Lemma~\ref{l2} to our analysis, with probability at least $1 - \delta - (M - 1) L d^{-10}$, we have

\begin{align}
\pR_\NT^2&=\sum_{m=2}^M \sum_{t=1}^T \sum_{n=\pG_{m-1}+1}^{\pG_m} x_\nt^{\star^\top} \thetats - x_\nt^\top \thetats\nonumber
\end{align}
\begin{align}
&=\sum_{m=2}^M \sum_{t=1}^T \sum_{n=\pG_{m-1}+1}^{\pG_m} x_\nt^{\star^\top} \thetats - x_\nt^\top \thetats+x_\nt^{\star^\top} \thetahatmt - x_\nt^{\star^\top} \thetahatmt\nonumber\\
&\leqslant \sum_{t=1}^T 2 \sqrt{(N - \pG_1) \log \frac{1}{\delta}} \|\thetats - \thetahatmt\|\nonumber\\
&\leqslant \sum_{t=1}^T 2 \sqrt{(N - \pG_1) \log \frac{1}{\delta}} \|\thetats - \thetahat_{1, t}\|\label{T10_2}\\
&\leqslant2\mu\sigma_{\max}^\star\sqrt{2(N-\pG_1)(\rho^2-\rho+1)rT\log\frac{1}{\delta}},\nonumber
\end{align}
where Eq.~\eqref{T10_2} is derived from $\|\thetats - \thetahatmt\| \leqslant \|\thetats - \thetahat_{1, t}\|$. Finally, by applying the union bound across all epochs, with probability at least $1 - 2 \delta - M L d^{-10}$, the cumulative regret $\pR_\NT$ is bound as follows: 
\begin{align}
\pR_\NT&=\sum_{m=1}^M \sum_{t=1}^T \sum_{n=\pG_{m-1}+1}^{\pG_m} x_\nt^{\star^\top} \thetats - x_\nt^\top \thetats\nonumber\\
&=\sum_{t=1}^T \sum_{n=1}^{\pG_1} x_\nt^{\star^\top} \thetats - x_\nt^\top \thetats+\sum_{m=2}^M \sum_{t=1}^T \sum_{n=\pG_{m-1}+1}^{\pG_m} x_\nt^{\star^\top} \thetats - x_\nt^\top \thetats\nonumber
\end{align}
\begin{align}
&\leqslant2 \mu \sigma_{\max}^\star \sqrt{\pG_1 r T (\rho^2 - \rho + 1) \log \frac{1}{\delta}}\nonumber\\
&+2 \mu \sigma_{\max}^\star \sqrt{2 (N - \pG_1) (\rho^2 - \rho + 1) r T \log \frac{1}{\delta}}\nonumber\\
&= 2 \mu \sigma_{\max}^\star \sqrt{r T (\rho^2 - \rho + 1) \log \frac{1}{\delta}} \left(\sqrt{\pG_1} + \sqrt{2 (N - \pG_1)}\right) \nonumber \\
&\leqslant4\sqrt{2}\mu\sigma_{\max}^\star\sqrt{rNT(\rho^2-\rho+1)\log\frac{1}{\delta}},\label{T10_3}
\end{align}
where Eq.~\eqref{T10_3} is derived from $\sqrt{x} + \sqrt{y} \leqslant 2 \sqrt{x + y}$. 
\qed

\begin{wrapfigure}{l}{25mm} 
\includegraphics[width=1in,height=1.25in,clip,keepaspectratio]{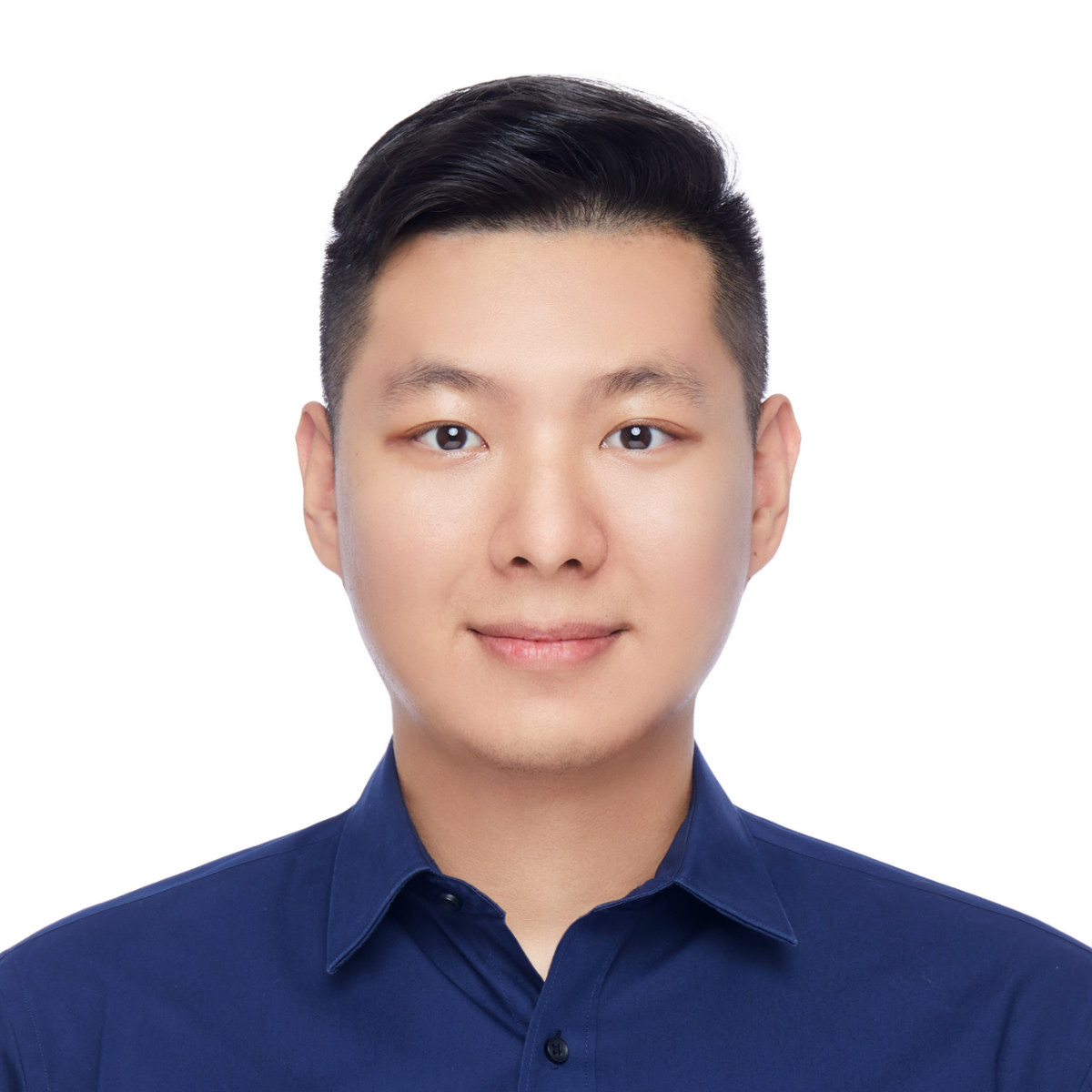}
\end{wrapfigure}\par
\textbf{Jiabin Lin} received his B.S. degree in Computer Science and Communication Engineering from Guangxi University of Science and Technology in 2019, and his M.S. degree in Electrical Engineering from the University of South Florida in 2021.  He is a graduate student in the Department of Electrical and Computer Engineering at Iowa State University. His research interests include bandit learning, multi-task learning, and representation learning.\par
\begin{wrapfigure}{l}{25mm} 
\includegraphics[width=1in,height=1.25in,clip,keepaspectratio]{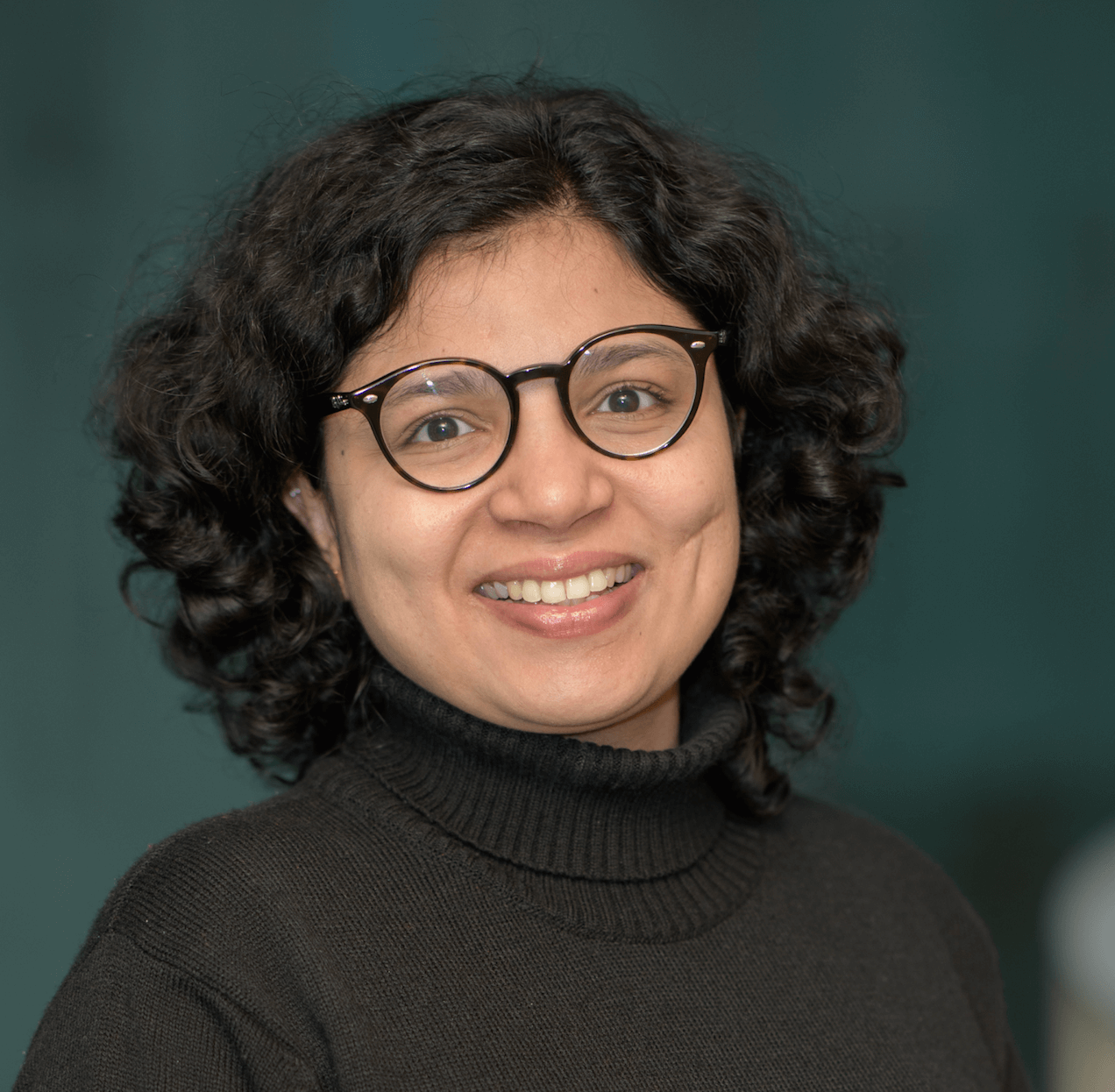}
\end{wrapfigure}\par
\textbf{Shana Moothedath} received her Ph.D. degree in Electrical Engineering from the Indian Institute of Technology Bombay (IITB) in 2018, and she was a postdoctoral scholar in Electrical and Computer Engineering at the University of Washington, Seattle (UW) till 2021. Currently, she is the Harpole-Pentair Assistant Professor in the Department of Electrical and Computer Engineering at Iowa State University. Her research focuses on distributed decision-making, learning and control, and security of cyber-physical systems. She was awarded ISU Award for Early Achievement in Research 2026, NSF CAREER Award 2025, the Best Research Thesis Award at IITB in 2019, and selected as a MIT-EECS Rising Star in 2019.\par

\end{document}